%% file: main.tex
\documentclass[nohyperref]{article}

\usepackage{microtype}
\usepackage{graphicx}
\usepackage{booktabs} 
\usepackage{enumitem}
\usepackage{xcolor}
\usepackage{multirow}
\usepackage{caption}
\usepackage{subcaption}

\usepackage{hyperref}

\usepackage[accepted]{icml2022}

\usepackage{amsmath}
\usepackage{amssymb}
\usepackage{mathtools}
\usepackage{amsthm}
\usepackage{macros}

\usepackage[capitalize,noabbrev]{cleveref}

\theoremstyle{plain}
\newtheorem{theorem}{Theorem}[section]
\newtheorem{proposition}[theorem]{Proposition}
\newtheorem{lemma}[theorem]{Lemma}

\theoremstyle{definition}
\newtheorem{definition}{Definition}[section]
\newtheorem{assumption}{Assumption}[section]
\newtheorem{condition}{Condition}[section]
\newtheorem{problem}{Problem}[section]
\newtheorem{example}{Example}[section]
\theoremstyle{remark}
\newtheorem{remark}{Remark}[section]

\usepackage[textsize=tiny]{todonotes}

\icmltitlerunning{A Theoretical Analysis on Independence-driven Importance Weighting for Covariate-shift Generalization}

\DeclareMathOperator{\var}{Var}
\DeclareMathOperator{\cov}{Cov}
\DeclareMathOperator{\mbd}{BD}
\DeclareMathOperator{\mbl}{BL}
\DeclareMathOperator{\pmbd}{MinStable}
\DeclareMathOperator{\pmbl}{Stable}
\DeclareMathOperator{\mbdtest}{\mbd^{test}}
\DeclareMathOperator{\mbltest}{\mbl^{test}}

\newcommand{\boldX}{\boldsymbol{X}}
\newcommand{\boldS}{\boldsymbol{S}}
\newcommand{\boldV}{\boldsymbol{V}}

\newcommand{\boldD}{\boldsymbol{D}}
\newcommand{\boldW}{\boldsymbol{W}}
\newcommand{\boldA}{\boldsymbol{A}}
\newcommand{\boldB}{\boldsymbol{B}}
\newcommand{\boldC}{\boldsymbol{C}}
\newcommand{\boldZ}{\boldsymbol{Z}}
\newcommand{\boldx}{\boldsymbol{x}}
\newcommand{\bolds}{\boldsymbol{s}}
\newcommand{\boldv}{\boldsymbol{v}}
\newcommand{\boldu}{\boldsymbol{u}}
\newcommand{\boldd}{\boldsymbol{d}}
\newcommand{\boldw}{\boldsymbol{w}}

\newcommand{\calX}{\mathcal{X}}
\newcommand{\calS}{\mathcal{S}}
\newcommand{\calV}{\mathcal{V}}

\newcommand{\calD}{\mathcal{D}}
\newcommand{\calW}{\mathcal{W}}
\newcommand{\calY}{\mathcal{Y}}
\newcommand{\bbE}{\mathbb{E}}
\newcommand{\bbR}{\mathbb{R}}
\newcommand{\tildeX}{\tilde{\boldsymbol{X}}}
\newcommand{\boldbeta}{\boldsymbol{\beta}}
\newcommand{\boldxi}{\boldsymbol{\xi}}
\newcommand{\appr}{\mathrm{approx}_w}
\newcommand{\noise}{\mathrm{noise}}
\newcommand{\pr}{\mathrm{Pr}}
\newcommand{\boundx}{B}
\newcommand{\boundappr}{C_w}
\newcommand{\boundlambda}{\Lambda_w}
\newcommand{\boundweight}{\delta_{\hat{w}}}
\newcommand{\minlambda}{\lambda_{\min}}

\newcommand{\ptrain}{P^{\text{tr}}}
\newcommand{\ptest}{P^{\text{te}}}

\newenvironment{enum}
{\begin{enumerate}[noitemsep,topsep=0pt,parsep=0pt,partopsep=0pt]}
{\end{enumerate}}

\makeatletter
\def\maketag@@@#1{\hbox{\m@th\normalfont\normalsize#1}}
\makeatother

\begin{document}

\twocolumn[
\icmltitle{A Theoretical Analysis on Independence-driven Importance Weighting for Covariate-shift Generalization}



\icmlsetsymbol{equal}{*}

\begin{icmlauthorlist}
\icmlauthor{Renzhe Xu}{tsinghua}
\icmlauthor{Xingxuan Zhang}{tsinghua}
\icmlauthor{Zheyan Shen}{tsinghua}
\icmlauthor{Tong Zhang}{hkust}
\icmlauthor{Peng Cui}{tsinghua}
\end{icmlauthorlist}

\icmlaffiliation{tsinghua}{Department of Computer Science and Technology, Tsinghua University, Beijing, China}
\icmlaffiliation{hkust}{Computer Science \& Mathematics, The Hong Kong University of Science and Technology, Hong Kong, China. Emails: xrz199721@gmail.com, xingxuanzhang@hotmail.com, shenzy17@mails.tsinghua.edu.cn, tongzhang@tongzhang-ml.org, cuip@tsinghua.edu.cn}

\icmlcorrespondingauthor{Peng Cui}{cuip@tsinghua.edu.cn}

\icmlkeywords{Stable Learning, Covariate-shift Generalization}

\vskip 0.3in
]



\printAffiliationsAndNotice{}  

\input{paragraphs/abstract.tex}

\input{paragraphs/intro}
\input{paragraphs/related_works}
\input{paragraphs/preliminaries}
\input{paragraphs/stable-set}
\input{paragraphs/algorithm}
\input{paragraphs/theories}
\input{paragraphs/experiments}
\input{paragraphs/discussions}

\input{paragraphs/acknowledgments}

\bibliographystyle{icml2022}
\bibliography{references}

\clearpage

\appendix
\onecolumn
\input{paragraphs/relationship}
\input{paragraphs/app-experiment}
\input{paragraphs/proofs}

\end{document}

%% file: paragraphs/abstract.tex
\begin{abstract}
    Covariate-shift generalization, a typical case in out-of-distribution (OOD) generalization, requires a good performance on the unknown test distribution, which varies from the accessible training distribution in the form of covariate shift. Recently, independence-driven importance weighting algorithms in stable learning literature have shown empirical effectiveness to deal with covariate-shift generalization on several learning models, including regression algorithms and deep neural networks, while their theoretical analyses are missing. In this paper, we theoretically prove the effectiveness of such algorithms by explaining them as feature selection processes. We first specify a set of variables, named \textbf{minimal stable variable set}, that is the minimal and optimal set of variables to deal with covariate-shift generalization for common loss functions, such as the mean squared loss and binary cross-entropy loss. Afterward, we prove that under ideal conditions, independence-driven importance weighting algorithms could identify the variables in this set. Analysis of non-asymptotic properties is also provided. These theories are further validated in several synthetic experiments. The source code is available at \url{https://github.com/windxrz/independence-driven-IW}.
\end{abstract}

%% file: paragraphs/intro.tex
\section{Introduction}
Although modern machine learning techniques have achieved great success in various areas, many researchers have demonstrated the vulnerability of machine learning models under distribution shifts \citep{shen2021towards}.
This issue arises from the violation of the \textit{i.i.d.} assumption (\textit{i.e.}, training and test data are independent and identically distributed) and stimulates recent research on out-of-distribution (OOD) generalization~\citep{shen2021towards,zhang2022nico++}.
Among different types of distribution shifts considered in OOD literature, covariate shift~\citep{shimodaira2000improving,sugiyama2007covariate,ben2007analysis}, where the marginal distribution of variables shifts from the training data to the test data while the labeling function keeps unchanged, is the most common one~\citep{shen2021towards}.
Further, covariate-shift generalization is much more challenging, given that the test distribution remains unknown in the training phase.

With the prior knowledge of the test distribution, importance weighting (IW) is common in dealing with covariate shift~\citep{shimodaira2000improving,sugiyama2007covariate,sugiyama2007direct,sugiyama2008direct,fang2020rethinking}.
In detail, IW methods consist of two steps, namely weight estimation and weighted regression~\citep{fang2020rethinking}. The weight estimation step estimates sample weights that characterize the density ratio between the training and test distribution. The weighted regression step trains predictors after plugging the sample weights into loss functions.
However, IW methods can not adapt to covariate-shift generalization problems directly because the test distribution is unknown.

Recently, independence-driven importance weighting methods \citep{shen2020stable,kuang2020balance,zhang2021deep,zhang2022towards2} in stable learning literature \citep{cui2022stable} have shown empirical effectiveness to deal with covariate-shift generalization on several learning tasks involving regression algorithms and deep models.
Without the knowledge of the test distribution, in the weight estimation step, they propose to learn sample weights that guarantee the statistical independence between features in the weighted distribution.
Although the advantages of these algorithms have been proved empirically, the theoretical explanations for these methods are missing. In this paper, we take a step towards the theoretical analysis of independence-driven IW methods on covariate-shift generalization problems by explaining them as feature selection processes.

We first show that for common loss functions, including the mean squared loss and binary cross-entropy loss, the covariate-shift generalization problem can be tackled by a minimal set of variables $\boldS$ that satisfies the condition: $\mathbb{E}[Y|\boldS] = \mathbb{E}[Y|\boldX]$.
Such a minimal set of variables is named the \textbf{minimal stable variable set}.
Afterward, we prove that independence-driven IW algorithms could identify the minimal stable variable set.
We analyze the typical algorithms \citep{kuang2020balance, shen2020stable} where the weighted least squares (WLS) is adopted in the weighted regression step.
Variables whose corresponding coefficients of WLS are not zero could be considered as chosen variables. 
Under ideal conditions, \textit{i.e.}, perfectly learned sample weights and infinite samples, the selected variables are proved to be the minimal stable variable set.
We further provide non-asymptotic properties and error analysis when the ideal conditions are not satisfied.
We highlight that although a linear model (WLS) is adopted, these theoretical results hold for both linear and non-linear data-generating processes.
Along with the optimality and minimality of the minimal stable variable set, these theories provide a way to explain why independence-driven IW methods work for covariate-shift generalization. These theories are further validated in several synthetic experiments.
 
\subsection{Overview of Results}
We begin with a simplified presentation of our results. Consider a set of variables $(\boldX, Y)$ where $\boldX$ represents features and $Y$ represents the outcome that we try to predict from $\boldX$. We consider covariate-shift generalization problems, which is the most common one among the different distribution shifts \citep{shen2021towards}. In detail, covariate shift considers the scenario where the marginal distribution of $\boldX$ shifts from the training phase to the test phase while the labeling function keeps unchanged.

\begin{assumption} \assumptionlabel{assum:covariate-shift}
    Suppose the test distribution $\ptest$ differs from the training distribution $\ptrain$ in covariate shift only, \textit{i.e.},
    \begin{equation}
        \ptest(\boldX, Y) = \ptest(\boldX)\ptrain(Y|\boldX).
    \end{equation}
    In addition, $\ptest$ has the same support of $\ptrain$.
\end{assumption}

\begin{problem} [Covariate-shift generalization problem] \problemlabel{prob:covariate}
    Given the samples from the training distribution $\ptrain$, covariate-shift generalization problem is to design an algorithm which can guarantee the performance on the unknown test distribution $\ptest$ that satisfies \assumptionref{assum:covariate-shift}.
\end{problem}

We focus on several common loss functions, including the mean squared loss and binary cross-entropy loss, under which circumstances
$\mathbb{E}_{\ptest}[Y|\boldX]$ is the global optimum for the test distribution $\ptest$.

\begin{theorem} [Informal version of \theoremref{thrm:ood-prediction}] \theoremlabel{thrm:informal-ood}
Let $\ptest$ be the unknown test distribution in the covariate-shift generalization problem defined in \problemref{prob:covariate}. Then a subset of variables $\boldS \subseteq \boldX$ that can fit the target $\mathbb{E}_{\ptest}[Y|\boldX]$ if and only if it satisfies $\mathbb{E}_{\ptrain}[Y | \boldS] = \mathbb{E}_{\ptrain}[Y | \boldX]$.
\end{theorem}

We define the minimal set of variables that satisfies $\mathbb{E}_{\ptrain}[Y | \boldS] = \mathbb{E}_{\ptrain}[Y | \boldX]$ as the \textbf{minimal stable variable set} (\definitionref{defn:minimal-stable-set}). Under mild assumptions (\assumptionref{assum:positive}), the existence and uniqueness of the variable set are guaranteed (\theoremref{thrm:unique-stable-set}). As relationships between $\boldX$ vary from the training phase to the test phase, \textit{i.e.}, $\ptrain(\boldX) \ne \ptest(\boldX)$, it is reasonable to find the minimal set of variables to make predictions so that it can relieve the negative impact of other features in the test distribution. We will show the optimality property of the minimal stable variable set empirically in \figureref{fig:ood-S-2.5}.

Now we consider independence-driven IW algorithms (The framework of such algorithms can be found in \sectionref{sect:stable-learning-algorithm} and \algorithmref{alg:stable-learning}). Typical independence-driven IW algorithms~\citep{shen2020stable,kuang2018stable} learn sample weights first to make features statistically independent in the weighted distribution and then adopt a weighted least squares regression step. The algorithms can be considered as processes of feature selection by examining the coefficients of WLS. In detail, the variables with non-zero coefficients are chosen. The variables chosen by independence-driven IW algorithms have the following properties.

\begin{theorem} [Informal version of \theoremref{thrm:findV} and \theoremref{thrm:findS}] \theoremlabel{thrm:informal-identifiabillity}
    Under ideal conditions (perfectly learned sample weights and infinite samples),
    \begin{itemize}[labelindent=0pt,noitemsep,topsep=0pt,parsep=0pt,partopsep=0pt]
        \item if a variable $X_i$ is not in the minimal stable variable set, then independence-driven IW algorithms could filter it out with any weighting function that satisfies the independence condition, and
        \item if a variable $X_i$ is in the minimal stable variable set, then there exist weighting functions with which independence-driven IW algorithms could identify $X_i$.
    \end{itemize}
\end{theorem}

We further analyze the error of coefficients if these ideal conditions are not satisfied (\theoremref{thrm:big-asumptotic}) under several mild assumptions.

\theoremref{thrm:informal-ood} and \theoremref{thrm:informal-identifiabillity} provide a general picture of the effectiveness of independence-driven IW algorithms. To conclude, under ideal assumptions, they could identify the minimal stable variable set, which is the minimal and optimal set of variables to deal with covariate-shift generalization.

%% file: paragraphs/related_works.tex
\subsection{Related Works} \sectionlabel{sect:related-works}
\paragraph{OOD and covariate-shift generalization} OOD generalization has raised great concerns. According to \citep{shen2021towards}, OOD methods could be categorized into unsupervised representation learning methods \citep{bengio2013representation,yang2021causalvae,zhang2022towards}, supervised learning models \citep{peters2016causal,zhou2021domain,liu2021heterogeneous,liu2021kernelized,zhou2022sparse,lin2022bayesian,lin2022zin}, and optimization methods \citep{duchi2020distributionally,duchi2021learning,zhou2022model}. More thorough discussions could refer to \citep{shen2021towards}.

There are many types of distribution shift, including covariate shift \citep{shimodaira2000improving}, label shift \citep{garg2020unified}, and concept shift \citep{gama2014survey} and covariate shift is the most common distribution shift \citep{shen2021towards}.
To deal with the covariate-shift generalization problem, there are several methods recently \citep{shen2020stable,kuang2020balance,zhang2021deep,duchi2021learning,krueger2021out,ruan2021optimal}. In this paper, we focus on independence-driven IW algorithms \citep{shen2020stable,kuang2020balance,zhang2021deep} and provide a theoretical analysis of them.

\paragraph{Importance weighting (IW) and independence-driven IW algorithms}~
Importance weighting methods are common practices to tackle distribution shifts. In traditional domain adaptation (DA) problems~\citep{daume2006domain,ben2007analysis}, importance weighting methods assume the prior knowledge of the test distribution and they can estimate the density ratio between the training and test distributions directly~\citep{shimodaira2000improving,huang2006correcting,storkey2007mixture,sugiyama2007covariate,sugiyama2007direct,bickel2007discriminative,sugiyama2008direct,kanamori2009least,fang2020rethinking}. As a result, the ERM training on the weighted distribution is unbiased in the test distribution~\citep{fang2020rethinking}.

Compared to typical DA settings, covariate-shift generalization problems consider a much more challenging setting where the test distribution is unknown \citep{shen2021towards}. Without the knowledge of the test distribution, independence-driven IW algorithms~\citep{shen2018causally,kuang2020stable,shen2020stable,zhang2021deep} in stable learning literature~\citep{cui2022stable} propose to learn sample weights that make features statistically independent in the weighted distribution.
Although the effectiveness of such algorithms on covariate-shift generalization has been proved empirically, their detailed theoretical analysis is missing.

\paragraph{Feature Selection}
Feature selection aims to construct a diagnostic or predictive model for a given regression or classification task via selecting a minimal-size subset of variables that show the best performance \citep{guyon2003introduction}.
Feature selection approaches can be broadly divided into four categories, namely filter methods, wrapper methods, embedded methods, and others.
Filter methods adopt statistical criteria to rank and select features before building classifiers with selected features \citep{john1994irrelevant,langley1994selection,guyon2003introduction, law2004simultaneous}.
Given filter methods are usually independent of the learning of the classifiers, they show superiority in operating time and applicability over other methods \citep{kira1992practical,bolon2013review}. Wrapper methods heuristically search variable subsets via learning a predictive model, thus they can identify the best performing feature subsets for the given modeling algorithm, but are typically computationally intensive \citep{menze2009comparison, bolon2013review, urbanowicz2018relief}. Embedded methods seek to minimize the size of the selected feature subset while maximizing the classification performance simultaneously \citep{tibshirani1996regression,rakotomamonjy2003variable,zou2005regularization,loh2011classification,chen2016xgboost}. Some methods attempt to combine the advantages of wrapper methods and filter methods \citep{cortizo2006multi, liu2014global, benoit2013feature}.
However, discussions on feature selection problems under covariate-shift generalization settings are missing. In this paper, we specify the optimal and minimal set of variables to deal with covariate-shift generalization and prove that independence-driven IW algorithms could identify them. 

%% file: paragraphs/preliminaries.tex
\section{Preliminaries} \sectionlabel{sect:preliminary}
\paragraph{Notations}
Let $\boldX = (X_1, X_2, \dots X_d)^T \in \mathbb{R}^d$ denote the $d$-dimensional features and $Y \in \mathbb{R}$ denote the outcome. The training data is from a joint training distribution $\ptrain(\boldX, Y)$. Let $\mathcal{X}$, $\mathcal{X}_j$, and $\mathcal{Y}$ denote the support of $\boldX$, $X_j$, and $Y$, respectively. Suppose we get $n$ \textit{i.i.d.} samples, $\left\{\boldx^{(i)} = \left(x_1^{(i)}, \dots, x_d^{(i)}\right)^T, y^{(i)}\right\}_{i=1}^n$ sampled from the distribution. Let $\ptest$ denote the unknown test distribution.

We use $\boldS \subseteq \boldX$ to indicate that $\boldS$ is a subset of features $\boldX$ and $\subsetneq$ to mean proper subset. We write $\boldA \perp \boldB \mid \boldC$ when two sets of variables $\boldA, \boldB \subseteq \boldX$ are statistically independent given another set of variables $\boldC \subseteq \boldX$. We also adopt $\boldA \perp \boldB$ when conditioning set is empty to indicate that $\boldA$ and $\boldB$ are statistically independent.

We use $\mathbb{E}_{Q(\cdot)}[\cdot]$ and $\mathbb{E}_{Q(\cdot)}[\cdot | \cdot]$ to denote expectation and conditional expectation, respectively, under a distribution $Q$. 
For example, $\mathbb{E}_{Q(\boldX)}[\boldX]=\int_{\calX}\boldx Q(\boldX=\boldx)\mathrm{d}x$ represent the expectation of $\boldX$ and $\mathbb{E}_{Q(\boldX,Y)}[Y|\boldX]=\int_{\calY}Q(Y=y|\boldX)y\mathrm{d}y$ represent the conditional expectation of $Y$ given $\boldX$ under distribution $Q$. 
$Q$ could be chosen as the training distribution $\ptrain$, test distribution $\ptest$, or any other proper distributions. 
If not confusing, we will use $\mathbb{E}[\cdot]$ and $\mathbb{E}[\cdot | \cdot]$ to denote the expectation and conditional expectation under the training distribution $\ptrain$. We use $\hat{\bbE}[\cdot]$ to denote the empirical expectation \textit{w.r.t.} $n$ samples.

\paragraph{Basic assumption} We consider the following assumption.
\begin{assumption} [Strictly positive density assumption] \assumptionlabel{assum:positive}
    $\forall x_1 \in \mathcal{X}_1, x_2 \in \mathcal{X}_2, \dots, x_d \in \mathcal{X}_d$, $\ptrain(X_1 = x_1, X_2 = x_2, \dots, X_d = x_d) > 0$.
\end{assumption}

\begin{remark}
    \assumptionref{assum:positive} is reasonable on the grounds that there always exists uncertainty in the data \citep{pearl2014probabilistic,strobl2016markov}. Therefore, we suppose the strictly positive density assumption in the whole paper for simplicity.
\end{remark}

%% file: paragraphs/stable-set.tex
\section{Minimal Stable Variable Set for Covariate-shift Generalization} \sectionlabel{sect:local-causal}
In this section, we specify the set of variables that are suitable for covariate-shift generalization problems. We first provide the definition of the minimal and optimal predictor.

\begin{definition} [Optimal predictor \citep{statnikov2013algorithms}]
    Given a dataset sampled from $\ptrain(\boldX, Y)$, a learning algorithm $\mathbb{L}$, and a performance metric $\mathbb{M}$ to assess learner’s models, a variable set $\boldS \subseteq \boldX$ is an optimal predictor of $Y$ if $\boldS$ maximizes the performance metric $\mathbb{M}$ for predicting $Y$ using learner $\mathbb{L}$ in the dataset.
\end{definition}

\begin{definition} [Minimal and optimal predictor \citep{strobl2016markov}]
    Let $\boldS$ be an optimal predictor of $Y$. If no proper subset of $\boldS$ satisfies the definition of the optimal predictor of $Y$, then $\boldS$ is a minimal and optimal predictor of $Y$.
\end{definition}

The minimal and optimal predictor for covariate-shift generalization can be given as follows.

\begin{theorem} \theoremlabel{thrm:ood-prediction}
    Under \assumptionref{assum:covariate-shift} and \assumptionref{assum:positive}, if $\mathbb{M}$ is a performance metric that is maximized only when $\mathbb{E}_{\ptest}[Y | \boldX]$ is estimated accurately and $\mathbb{L}$ is a learning algorithm that can approximate any conditional expectation. Suppose $\boldS \subseteq \boldX$ is a subset of variables, then
    \begin{enum}
        \item $\boldS$ is an optimal predictor of $Y$ under distribution $\ptest$ if and only if $\mathbb{E}_{\ptrain}[Y| \boldX] =\mathbb{E}_{\ptrain}[Y | \boldS]$, and
        \item $\boldS$ is a minimal and optimal predictor of $Y$ under distribution $\ptest$ if and only if $\mathbb{E}_{\ptrain}[Y| \boldX] =\mathbb{E}_{\ptrain}[Y | \boldS]$ and no proper subset $\boldS' \subsetneq \boldS$ satisfies $\mathbb{E}_{\ptrain}[Y | \boldX] = \mathbb{E}_{\ptrain}[Y | \boldS']$.
    \end{enum}
\end{theorem}
\begin{remark}
    To deal with covariate-shift generalization, $\mathbb{M}$ should be measured on the unknown test distribution $\ptest$ with common loss functions.
    In practice, researchers often adopt the mean squared loss in regression problems and the binary cross-entropy loss in binary classification problems. It is easy to check that the global optimum for both loss functions is $\mathbb{E}_{\ptest}[Y|\boldX]$ if applying the loss functions on the test distribution $\ptest$.
\end{remark}

As a result, we provide the following definitions.
\begin{definition} [Stable variable set] \definitionlabel{defn:stable-set}
    A stable variable set of $Y$ under distribution $P$ is any subset $\boldS$ of $\boldX$ for which
    \begin{equation} \equationlabel{eq:stable-set}
        \mathbb{E}_P[Y | \boldS] = \mathbb{E}_P[Y | \boldX].
    \end{equation}
    The set of all stable variable sets for $Y$ is denoted as $\pmbl_P(Y)$. In addition, we use $\pmbl(Y)$ to denote the set under the training distribution $\ptrain$ for simplicity, \textit{i.e.}, $\pmbl(Y) \triangleq \pmbl_{\ptrain}(Y)$.
\end{definition}

\begin{definition} [Minimal stable variable set] \definitionlabel{defn:minimal-stable-set}
    A minimal stable variable set of $Y$ is a minimal set in $\pmbl(Y)$, \textit{i.e.}, none of its proper subsets satisfies \equationref{eq:stable-set}.
\end{definition}

With these definitions, the conclusions of \theoremref{thrm:ood-prediction} become: (1) $\boldS$ is an optimal predictor of $Y$ under $\ptest$ if and only if it is a stable variable set under $\ptrain$, and (2) $\boldS$ is a minimal and optimal predictor of $Y$ under $\ptest$ if and only if it is a minimal stable variable set under $\ptrain$. Furthermore, the existence and uniqueness of the minimal stable variable set are given by the following theorem.
\begin{theorem} 
\theoremlabel{thrm:unique-stable-set}
    Under \assumptionref{assum:positive}, there exists a unique minimal stable variable set of $Y$, which can be denoted as $\pmbd(Y)$.
    Furthermore, with the unique $\pmbd(Y)$, the set of all stable variable sets of $Y$ under the training distribution $\ptrain$, \textit{i.e.}, $\pmbl(Y)$, is
    \begin{equation}
        \pmbl(Y) = \{\boldS \subseteq \boldX \mid \pmbd(Y) \subseteq \boldS\}.
    \end{equation}
\end{theorem}

\theoremref{thrm:ood-prediction} and \theoremref{thrm:unique-stable-set} provide a way to ensure promising OOD performance for covariate-shift generalization problems. 
The minimal stable variable set under the training distribution $\ptrain$ is a minimal and optimal predictor in the test distribution $\ptest$, with which we can learn reliable models \citep{john1994irrelevant, guyon2003introduction}. 
As relationships between $\boldX$ are usually unstable and $\ptrain(\boldX) \ne \ptest(\boldX)$, it is reasonable to find the minimal and optimal predictor, \textit{i.e.}, $\pmbd(Y)$, to make predictions so that it can relieve the negative impact from $\boldX \backslash \pmbd(Y)$ under the test distribution.

\paragraph{Comparing the minimal stable variable set with other variable sets}
$\pmbd(Y)$ could be explained as the direct causal variables in typical data-generating processes. Consider the following mechanism  \citep{tibshirani1996regression,ravikumar2009sparse,hastie2017generalized,kuang2020stable},
\begin{equation}
    \boldX = (\boldS, \boldV), \quad Y = f(\boldS) + \epsilon, \quad \epsilon \perp \boldX.
\end{equation}
Here variables $\boldX$ contain two kinds of variables ($\boldS$ and $\boldV$) while $Y$ depends on $\boldS$ only. The relationship between $\boldS$ and $\boldV$ is arbitrary. In such common cases, $\boldS$ is the set of all the direct causal variables and is the minimal stable variable set of $Y$.

In addition, the minimal stable variable set has relationships with the stable blanket proposed by \citet{pfister2021stabilizing}. However, the stable blankets are defined in causal graphs over a set of interventions while the minimal stable variable set targets for the covariate-shift generalization.

Furthermore, the minimal stable variable set is closely related to the Markov boundary \citep{pearl2014probabilistic}. Under the performance metric in \theoremref{thrm:ood-prediction}, the minimal stable variable set shares the same prediction power of $Y$ with the Markov boundary while the minimal stable variable set contains fewer variables and thus combats covariate-shift generalization problems better. A detailed comparison between the minimal stable variable set and the Markov boundary can be found in \appendixref{sect:markov-boundary}.

%% file: paragraphs/algorithm.tex
\section{Independence-driven IW Algorithms} \sectionlabel{sect:stable-learning-algorithm}
\subsection{General Framework}
The framework of typical independence-driven importance weighting algorithms~\citep{shen2020stable,kuang2020stable} is shown in \algorithmref{alg:stable-learning}. Similar to standard IW algorithms~\citep{fang2020rethinking}, independence-driven IW algorithms consist of two steps, which are independence-driven weight estimation and weighted least squares respectively.

\subsubsection{Independence-driven Weight Estimation}
Independence-driven IW algorithms consider weighting functions that depend on $\boldX$ only.

\begin{definition} [Weighting function and weighted distribution] \definitionlabel{defn:weighting}
    Let $\mathcal{W}$ be the set of \textbf{weighting functions} that satisfies
    \begin{equation}
        \mathcal{W} = \left\{w: \mathcal{X} \rightarrow \mathbb{R}^{+} \mid \mathbb{E}_{\ptrain}[w(\boldX)] = 1 \right\}.
    \end{equation}
    Then $\forall w \in \mathcal{W}$, the corresponding \textbf{weighted distribution} $\tilde{P}_w$ can be determined by the following probability density function.
    \begin{equation}
        \tilde{P}_w(\boldX, Y) = w(\boldX)\ptrain(\boldX, Y).
    \end{equation}
    $\tilde{P}_w$ is well defined with the same support of $\ptrain$.
\end{definition}

Furthermore, instead of the whole set $\mathcal{W}$, independence-driven IW algorithms consider a subset $\mathcal{W}_{\perp} \subseteq \mathcal{W}$. The weighting functions in $\mathcal{W}_{\perp}$ satisfies that $\boldX$ are mutually independent of each other in the corresponding weighted distribution $\tilde{P}_w$ and the expectation of $\boldX$ in the weighted distribution is $\boldsymbol{0}$, \textit{i.e.},
\begin{equation}
    \small
    \begin{aligned}
        \mathcal{W}_{\perp} \triangleq \Huge\{ w \in \mathcal{W} \mid & \boldX \text{ are statistically independent in } \tilde{P}_w, \\
        & \bbE_{\ptrain}\left[w(\boldX)\boldX\right] = \boldsymbol{0} \Huge\}.
    \end{aligned}
\end{equation}

\def\NoNumber#1{{\def\alglinenumber##1{}\State #1}\addtocounter{ALG@line}{-1}}
\begin{algorithm}[tb]
    \caption{Independence-driven IW Algorithm}
    \algorithmlabel{alg:stable-learning}
    \begin{algorithmic}[1]
        \STATE {\bfseries Input:} Dataset $\left\{\boldx^{(i)} = \left(x_1^{(i)}, \dots, x_d^{(i)}\right)^T, y^{(i)}\right\}_{i=1}^n$
        \STATE Learn sample weight $w \in \calW_{\perp}$ so that $\boldX$ are statistically independent in the weighted distribution $\tilde{P}_w$.
        \STATE Solve weighted least squares with weighting function $w(\boldX)$. The solution is $\hat{\boldbeta}_w$.
        \STATE {\bfseries Output:} Cofficients of weighted least squares $\hat{\boldbeta}_w$.
    \end{algorithmic}
\end{algorithm}





\subsubsection{Weighted Least Squares}
Let $w \in \mathcal{W}$ be a weighting function. With $n$ datapoints sampled from $\ptrain(\boldX, Y)$, the weighted least squares solves the following equation
\begin{equation} \equationlabel{eq:weight-LS}
    \small
    \hat{\boldbeta}_w  = \argmin_{\boldbeta} \hat{\mathbb{E}}\left[ w\left(\boldX\right)\left(\boldbeta^T\boldX - Y\right)^2\right] = \hat{\Sigma}_w^{-1}\hat{\bbE}[w(\boldX)\boldX Y].
\end{equation}
Here $\hat{\Sigma}_w \triangleq \hat{\bbE}[w(\boldX)\boldX\boldX^T]$ represents the empirical covariance matrix with sample weights $w$. Furthermore, we denote the solution to population level weighted least squares under distribution $\ptrain(\boldX, Y)$ as
\begin{equation} \equationlabel{eq:weight-LS-population}
    \small
    \boldbeta_w = \argmin_{\boldbeta} \mathbb{E}\left[ w\left(\boldX\right)\left(\boldbeta^T\boldX - Y\right)^2\right] = \Sigma_w^{-1}\bbE[w(\boldX)\boldX Y].
\end{equation}
Here $\Sigma_w \triangleq \bbE[w(\boldX)\boldX\boldX^T]$ represents the population level covariance matrix. In addition, we use $\boldbeta_w(X_i)$ and $\hat{\boldbeta}_w(X_i)$ to denote the corresponding coefficient of $\boldbeta_w$ and $\hat{\boldbeta}_w$ on the $i$-th feature $X_i$.

\subsection{Two Specific Implementations} \sectionlabel{sect:specific}
\algorithmref{alg:stable-learning} has two typical implementations, namely DWR \citep{kuang2020stable} and SRDO \citep{shen2020stable}. They differ mainly in the way to learn sample weights $w$.

\paragraph{DWR} \citet{kuang2020stable} propose to decorrelate every two features, \textit{i.e.},
\begin{equation} \equationlabel{eq:DWR}
    w(\boldX) = \arg \min_{w_0(\boldX)} \sum_{1\le i,j \le d, i\ne j}\left(\cov(X_i, X_j; w_0)\right)^2,
\end{equation}
where $\cov(X_i, X_j; w_0)$ represents the covariance of features $X_i$ and $X_j$ in the weighted distribution $\tilde{P}_{w_0}$. The loss function in \equationref{eq:DWR} focuses on the linear correlation only and is used as an approximation for statistical independence. They proved that linear decorrelation suffices to generate good prediction models under simple models. Recently, \citet{zhang2021deep} combined DWR with random fourier features \citep{rahimi2007random} to achieve statistical independence and showed that deep models could perform better if the representations are statistically independent instead of linearly decorrelated.

\paragraph{SRDO} \citet{shen2020stable} propose to learn $w(\boldX)$ by estimating the density ratio of the training distribution $\ptrain$ and a specific weighted distribution $\tilde{P}$. The weighted distribution $\tilde{P}$ is determined by performing random resampling on each feature so that $\tilde{P}(X_1, X_2, \dots, X_d) = \prod_{i=1}^d\ptrain(X_i)$.
As a result, the weighting function $w(\boldX)$ is given by
\begin{equation} \equationlabel{eq:SRDO}
    w(\boldX) = \frac{\tilde{P}(\boldX)}{\ptrain(\boldX)} = \frac{\ptrain(X_1)\ptrain(X_2)\cdots\ptrain(X_d)}{\ptrain(X_1, X_2, \dots, X_d)}.
\end{equation}
The density ratio in \equationref{eq:SRDO} can be tackled by class-probability estimation problems and can be learned by several methods such as the binary cross-entropy loss, the LSIF loss \citep{kanamori2009least}, and the KLIEP loss \citep{sugiyama2009density}. A thorough review of density ratio estimation methods is presented by \citet{menon2016linking}. As a result, SRDO can guarantee statistical independence between variables $\boldX$ if the density ratio is estimated accurately.

%% file: paragraphs/theories.tex
\section{Theoretical Analysis of Independence-driven IW Algorithms} \sectionlabel{sect:theory}
In this section, we will show that independence-driven IW algorithms as shown in \algorithmref{alg:stable-learning} can be considered as a process of feature selection according to the coefficients of weighted least squares. The chosen features are the minimal stable variable set in \definitionref{defn:minimal-stable-set}. We first show the identifiability result with perfectly learned weighting functions and infinite samples in \sectionref{section:population-level}. Afterward, we relax the assumptions and study the non-asymptotic properties in \sectionref{section:asymptotic}. These theoretical results, along with \theoremref{thrm:ood-prediction} could prove the effectiveness of independence-driven IW algorithms for the covariate-shift generalization problem (\problemref{prob:covariate}).

\subsection{Population Level Properties} \sectionlabel{section:population-level}
Generally speaking, with infinite samples, for any perfectly learned proper weighting function $w \in \mathcal{W}_{\perp}$ adopted by the algorithms, the coefficient on variables that do not belong to the minimal stable variable set will be zero (\theoremref{thrm:findV}). In addition, there exist proper weighting functions with which the coefficients on the minimal stable variable set would not be zero (\theoremref{thrm:findS}).

\begin{theorem} \theoremlabel{thrm:findV}
    Under \assumptionref{assum:positive}, suppose $X_i \not\in \pmbd(Y)$. Let $w$ be any weighting function in $\mathcal{W}_\perp$. Suppose $\mathbb{E}_{\ptrain(\boldX)}\left[w(\boldX)\lVert \boldX\rVert_2^2\right] < \infty$ and $\mathbb{E}_{\ptrain(\boldX, Y)}\left[w(\boldX)Y^2\right] < \infty$. Then the population level solution $\boldbeta_w$ of weighted least squares under $w$ satisfies $\boldbeta_w(X_i) = 0$.
\end{theorem}

\begin{theorem} \theoremlabel{thrm:findS}
    Under \assumptionref{assum:positive}, suppose $X_i \in \pmbd(Y)$. Then there exists $w \in \mathcal{W}_\perp$ and constant $\alpha \ne 0$, such that the population level solution $\boldbeta_w$ satisfies $\boldbeta_w(X_i) = \alpha$.
\end{theorem}

\begin{remark}
    In very rare cases, independence-driven IW algorithms may fail to identify the minimal stable variable set if $X_i$ is not independent of $Y$ but is linearly decorrelated with $Y$ in the weighted distribution $\tilde{P}_w$.
\end{remark}

These two theorems, along with \theoremref{thrm:ood-prediction}, prove the effectiveness of independence-driven IW algorithms for the covariate-shift generalization problem (\problemref{prob:covariate}). In detail, under ideal conditions, \textit{i.e.}, perfectly learned sample weights and infinite samples, independence-driven IW algorithms could find the minimal stable variable set of $Y$, which is the minimal and optimal predictor under the test distribution $\ptest$ according to \theoremref{thrm:ood-prediction}.

\subsection{Non-asymptotic Properties} \sectionlabel{section:asymptotic}
We further analyze the non-asymptotic properties of independence-driven IW algorithms in this subsection. Given a weighting function $w \in \calW$, let
\begin{equation}
    \left\{
    \begin{aligned}
        \appr(\boldX) & \triangleq \bbE[Y | \boldX] - \langle\boldbeta_w, \boldX\rangle, \\
        \noise(\boldX) & \triangleq Y - \bbE[Y | \boldX].
    \end{aligned}
    \right.
\end{equation}

Here $\appr(\boldX)$ denotes the model misspecification term \textit{w.r.t.} linear models and $\noise(\boldX)$ represents the noise term of $Y$. For a non-trivial non-asymptotic property of the independence-driven IW algorithms, similar to \citet{zhang2005learning,hsu2014random}, we first make the assumptions about the data-generating process between $\boldX$ and $Y$.

\begin{assumption} [Bounded covariate] \assumptionlabel{assum:bound-covariate}
    There exists a finite constant $\boundx > 0$ such that, in the training distribution $\ptrain$, almost surely, $\|\boldX\|_2 \le \boundx$.
\end{assumption}

\begin{assumption} [Bounded approximation error] \assumptionlabel{assum:bound-appr}
    There exists a finite constant $\boundappr > 0$ such that, in the training distribution $\ptrain$, almost surely, $|\appr(\boldX)| \le \boundappr$.
\end{assumption}

\begin{assumption} [Sub-gaussian noise] \assumptionlabel{assum:noise}
    There exists a finite constant $\sigma \ge 0$ such that, in the training distribution $\ptrain$, almost surely, $\forall \eta \in \bbR$, $\bbE\left[\left. \exp\left(\eta \cdot \noise(\boldX)\right) \right| \boldX\right] \le \exp\left(\eta^2\sigma^2/2\right)$.
\end{assumption}

Furthermore, we assume that the chosen weighting function is non-degenerate.

\begin{assumption} [Non-degenerate weighting function] \assumptionlabel{assum:target-weighting}
    The minimal eigenvalue of $\Sigma_w$ is greater than $0$, \textit{i.e.}, $\lambda_{\min}\left(\Sigma_w\right) \triangleq \boundlambda > 0$.
\end{assumption}

In practice, we can not obtain the true weighting function $w$ and we need to estimate it from finite samples. The estimated weighting function is denoted as $\hat{w}$. We further provide assumptions about it.

\begin{assumption} [Small estimation error of the weighting function] \assumptionlabel{assum:estimate-weighting}
    The estimation error of the estimated weighting function $\hat{w}$ is small. In detail, $\bbE\left[\left(w(\boldX) - \hat{w}(\boldX)\right)^2\right] \triangleq \epsilon^2 < \boundlambda^2 /\bbE[\|\boldX\|_2^4]$.
\end{assumption}
\begin{assumption} [Bounded estimated weighting function] \assumptionlabel{assum:bound-weight}
    There exists a finite constant $\boundweight > 0$ such that, in the training distribution $\ptrain$, almost surely, $\hat{w}(\boldX) < \boundweight$.
\end{assumption}

\begin{remark}
    To ensure a small $\epsilon^2$ in \assumptionref{assum:estimate-weighting}, we can adopt LSIF \citep{kanamori2009least} to optimize $\mathbb{E}_{\ptrain(\boldX)}[(w(\boldX) - \hat{w}(\boldX))^2]$ directly. If we know a weighted distribution $Q$ and want to learn a weighting function $w(\boldX) = Q(\boldX)/\ptrain(\boldX)$. According to \citet{menon2016linking}, the loss of LSIF is $ L(w) = \mathbb{E}_{Q(\boldX)}[-w(\boldX)] + \mathbb{E}_{\ptrain(\boldX)}\left[w(\boldX)^2/2\right]$.
    It is easy to see that $w^*(\boldX) = \min_w L(w) = Q(\boldX) / \ptrain(\boldX)$ and $L(w) - L(w^*) = \mathbb{E}_{\ptrain(\boldX)}\left[(w^*(\boldX) - w(\boldX))^2\right] / 2$. As a result, minimizing the loss of LSIF will meet the assumption which requires that $\mathbb{E}_{\ptrain(\boldX)}[(w(\boldX) - \hat{w}(\boldX))^2] = \epsilon^2$ be small enough.
\end{remark}

\begin{remark}
    The difference between SRDO and DWR lies in \assumptionref{assum:estimate-weighting} due to the way of learning sample weights. Specifically, it is harder for DWR to satisfy \assumptionref{assum:estimate-weighting} because DWR focuses more on the linear correlation. As a result, its performance may drop when $Y$ has a complex non-linear relationship with $\mathbf{X}$, which is further validated by our experiments as shown in the fourth point in \sectionref{sect:exp-result}.    
\end{remark}

With the assumptions, we can provide the non-asymptotic property of independence-driven IW algorithms.

\begin{theorem} \theoremlabel{thrm:big-asumptotic}
    Let $w \in \calW$ be a weighting function. Suppose \assumptionsref{assum:positive}, \ref{assum:assum:bound-covariate}-\ref{assum:assum:bound-weight} (with parameters $\boundx$, $\boundappr$, $\sigma$, $\boundlambda$, $\epsilon$, $\boundweight$) hold. Pick any $t > \max\{0, 2.6 - \log d\}$, let
    \begin{equation}
        n \ge \frac{6\boundweight \boundx^2 (\log d + t)}{\boundlambda - \epsilon \sqrt{\mathbb{E}\left[\|\boldX\|_2^4\right]}}.
    \end{equation}
    Then with probability at least $1 - 3e^{-t}$,
    \begin{small}
        \begin{equation} \equationlabel{eq:big-bound}
            \begin{aligned}
                & \left\|\hat{\boldbeta}_{\hat{w}} - \boldbeta_{w}\right\|_2^2 \\
                \le & \underbrace{\frac{4\boundweight\sigma^2(d + 2\sqrt{td} + 2t)}{n\left(\boundlambda - \epsilon \sqrt{\mathbb{E}\left[\|\boldX\|_2^4\right]}\right)} + \frac{8\boundweight\boundx^2\boundappr^2(1+\epsilon)\left(1+\sqrt{8t}\right)^2}{n\left(\boundlambda - \epsilon \sqrt{\mathbb{E}\left[\|\boldX\|_2^4\right]}\right)^2}}_{\text{error caused by WLS from finite samples}} \\
                + & \underbrace{\frac{4\epsilon^2M_w}{\left(\boundlambda - \epsilon \ \sqrt{\mathbb{E}\left[\|\boldX\|_2^4\right]}\right)^2}}_{\text{error caused by imperfectly learned weights}} + \, o(1 / n).
            \end{aligned}
        \end{equation}
    \end{small}
    Here $M_w = \|\Sigma_w\|_2^2\|\boldbeta_w\|_2^2 \left(\frac{\mathbb{E}\left[\|\boldX\|_2^4\right]}{\|\Sigma_w\|_2^2}+\frac{\mathbb{E}\left[\|\boldX Y\|_2^2\right]}{\|\bbE[w(\boldX)\boldX Y]\|_2^2}\right)$ is a constant when $w$ is fixed. In particular, if $w \in \calW_{\perp}$, then
    \begin{equation}
        \left\|\hat{\boldbeta}_{\hat{w}} - \boldbeta_{w}\right\|_2^2 = \left\|\hat{\boldbeta}_{\hat{w}}(\boldV)\right\|_2^2 + \left\|\hat{\boldbeta}_{\hat{w}}(\boldS) - \boldbeta_{w}(\boldS)\right\|_2^2,
    \end{equation}
    where $\hat{\boldbeta}_{\hat{w}}(\boldV)$ represents the coefficients of $\hat{\boldbeta}_{\hat{w}}$ on $\boldX \backslash \pmbd(Y)$ and $\hat{\boldbeta}_{\hat{w}}(\boldS)$, $\boldbeta_{\hat{w}}(\boldS)$ represent the coefficients of $\hat{\boldbeta}_{\hat{w}}$, $\boldbeta_w$ on $\pmbd(Y)$. As a result, $\left\|\hat{\boldbeta}_{\hat{w}}(\boldV)\right\|_2^2$ and $\left\|\hat{\boldbeta}_{\hat{w}}(\boldS) - \boldbeta_{w}(\boldS)\right\|_2^2$ are also bounded by the RHS of \equationref{eq:big-bound}.
\end{theorem}
\begin{remark}
    \equationref{eq:big-bound} applies for any weighting function $w \in \calW$ that satisfies the listed assumptions. Excluding the high-order term of $o(1/n)$, the RHS of \equationref{eq:big-bound} consists of two parts. The first part is caused by WLS from finite samples and it vanishes when $n \rightarrow \infty$. The second part is caused by the error between the estimated function $\hat{w}$ and the true weighting function $w$ and it also vanishes when $\epsilon \rightarrow 0$.
    
    In particular, let $w \in \calW_{\perp}$ be a weighting function adopted by independence-driven IW algorithms. According to \theoremsref{thrm:findV} and \ref{thm:thrm:findS}, the coefficients on $\boldX \backslash \pmbd(Y)$ (\textit{i.e.}, $\|\hat{\boldbeta}_{\hat{w}}(\boldV)\|_2^2$) and the error of coefficients on $\pmbd(Y)$ (\textit{i.e.}, $\|\hat{\boldbeta}_{\hat{w}}(\boldS) - \boldbeta_{w}(\boldS)\|_2^2$) will be bounded by the RHS of \equationref{eq:big-bound} and become zero when $n \rightarrow \infty$ and $\epsilon \rightarrow 0$. This property guarantees that we could eliminate $\boldX \backslash \pmbd(Y)$ and find $\pmbd(Y)$ with finite samples and imperfectly learned sample weights.
\end{remark}

%% file: paragraphs/experiments.tex
\section{Synthetic Experiments}
We run various experiments on synthetic data to verify the effectiveness of independence-driven IW algorithms in discovering the minimal stable variable set in covariate-shift generalization problems. We consider the following data-generating process similar to \citet{kuang2020stable}.
\begin{figure*}[t]
    \centering
    \includegraphics[width=0.95\linewidth]{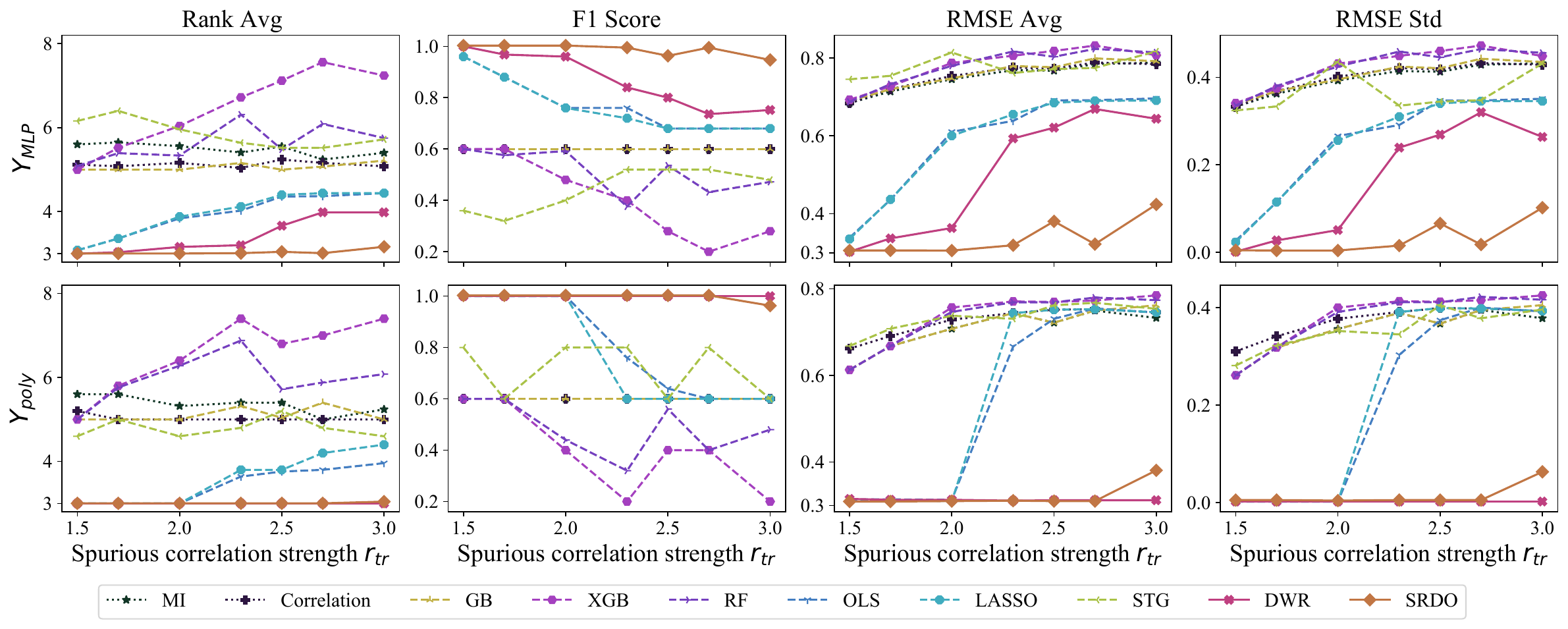}
    \vspace{-10px}
    \caption{Experimental results on synthetic data (MLP non-linear function $Y_{\text{MLP}}$ and polynomial function $Y_{\text{poly}}$ from top to bottom). Varying the spurious correlation strength $r_{\text{tr}}$, we compare independence-driven IW algorithms (DWR and SRDO shown in solid lines) with several baselines (shown in dashed lines) on both feature selection (Rank average and F1 score) and covariate-shift generalization (RMSE average and standard deviation) metrics. Independence-driven IW algorithms outperform other methods in the synthetic experiments.}
    \figurelabel{fig:experiment}
    \vspace{-10px}
\end{figure*}

\begin{figure} [th]
    \centering
    \includegraphics[width=\linewidth]{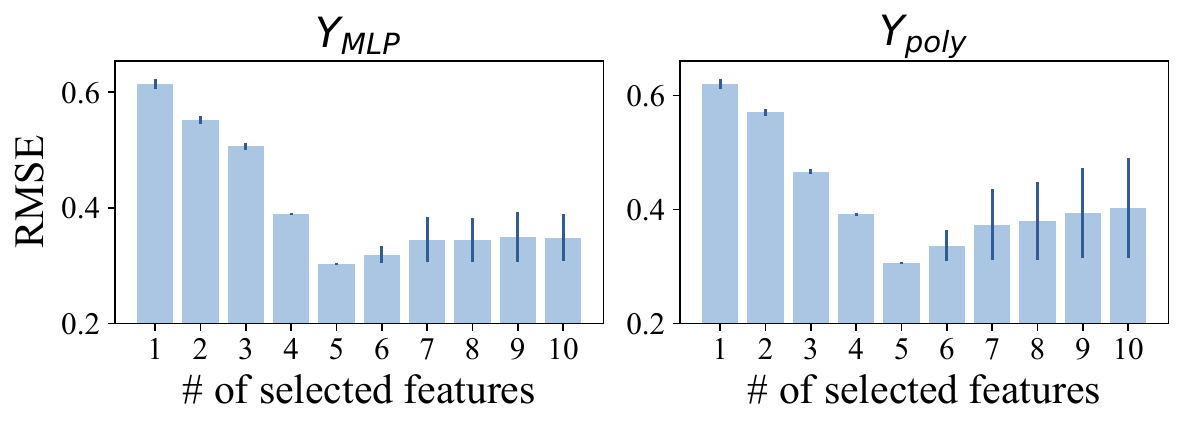}
    \vspace{-25px}
    \caption{The covariate-shift generalization metrics (RMSE average and standard deviation) \textit{w.r.t.} the number of selected features. Fix $r_{\text{tr}}=2.5$ here and the feature ranking lists are provided by SRDO. The minimal stable variable set (5 features) achieves the optimal performance.}
    \figurelabel{fig:ood-S-2.5}
    \vspace{-10px}
\end{figure}

\subsection{Data-generating Process}
\paragraph{Data} 
Let $\boldX=(\boldS, \boldV)$ and the dimension of $\boldX$ is fixed to $d = 10$. 
In our experiments, the dimensions of $S$ and $V$ are specified as $d_s = d_v = 0.5 \cdot d = 5$. 
Covariate $\boldX$ is generated by the following process.
\begin{equation}
    \begin{aligned}
        & Z_1, Z_2, \dots, Z_{d + 1} \sim N(0, 1),  V_1, V_2, \dots, V_{d_v} \sim N(0,1), \\
        & S_{i} = 0.8Z_i + 0.2Z_{i+1}, i = 1, 2, \dots, d_s. \\
    \end{aligned}
\end{equation}
We further clip features $\boldX$ into $[-2, 2]$ by letting $\boldX = \max(\min(\boldX, 2), -2)$. 
The outcome Y is generated through $Y = f(\boldS) + \epsilon$, in which $f(\boldS)$ may contain both linear and non-linear transformations.
To test the performance with different forms of non-linear terms in $f(\boldS)$, we generate the outcome Y from an MLP non-linear function ($Y_{\text{MLP}}$) and a polynomial one ($Y_{\text{poly}}$), respectively:
\begin{equation}
    \left\{
    \begin{aligned}
        Y_{\text{MLP}} & = f(\boldS) + \epsilon = \boldbeta^T\boldS + \text{MLP}([S_1, S_2, S_3]; \theta) + \epsilon, \\
        Y_{\text{poly}} & = f(\boldS) + \epsilon = \boldbeta^T\boldS + S_1S_2S_3/4 + \epsilon.
    \end{aligned}
    \right.
\end{equation}
Here $\boldbeta=\{1/3, -2/3, 1, -1/3, 2/3\}$, $\epsilon \sim N(0, 0.3^2)$, and $\text{MLP}([S_1, S_2, S_3]; \theta)$ represents the transformation of MLP with two hidden layers (sizes $3$ and $3$, respectively) parametrized by randomly generated $\theta \sim U(-1, 1)$.

\paragraph{Generating various environments} We generate various environments by constructing spurious correlations between $Y$ and $V_4$, $V_5$. Specifically, we fix a bias rate $r \in \bbR$ $(|r| > 1)$ in each generated environment. For each sample $\boldx^{(i)} = (\bolds^{(i)}, \boldv^{(i)})$, we select it to the corresponding environment with probability $\pr\left(\text{select}|\boldx^{(i)}; r\right) = \prod_{j=4}^5|r|^{-10D_j^{(i)}}$, where $D_j^{(i)} = \left|f(\boldS) - \text{sgn}(r) * \boldv^{(i)}_j\right|$ and $\text{sgn}(r)$ is the indicator function on whether $r > 0$. Intuitively, $r$ controls the strength and direction of spurious correlations. Specifically, $r>1$ corresponds to the positive spurious correlation between $Y$ and $\boldV$ and $r<-1$ corresponds to the negative spurious correlation. In addition, the higher $|r|$ is, the stronger correlation between $Y$ and $\boldV$ becomes.

Here $P(Y | \boldX)$ is 
obviously invariant across different environments and the data-generating process satisfies the covariate-shift condition. Moreover, the minimal stable variable set $\pmbd(Y)$ is $\boldS$ in each environment.

\vspace{-10px}

\paragraph{Experimental setting} In the $Y_{\text{MLP}}$ setting, we randomly generate $5$ different MLPs and report the results averaged over the $5$ MLPs. We train feature selection models on one training dataset with a specific bias rate $r_{\text{tr}}$ and $n=10,000$ samples. We then choose the top 5 features selected by each model and further train an MLP regressor on them. The regressor is then evaluated on $10$ test environments with corresponding $r_{\text{te}} = -3.0$, $-2.5$, $-2.0$, $-1.5$, $-1.3$, $1.3$, $1.5$, $2.0$, $2.5$, $3.0$. To test the effect of spurious correlation strength on feature selection models, we vary $r_{\text{tr}} = 1.5$, $1.7$, $2.0$, $2.3$, $2.5$, $2.7$, $3.0$.
    
\subsection{Baselines and Evaluation Metrics}
\paragraph{Baselines}
We compare independence-driven IW algorithms (including \textbf{DWR}~\citep{kuang2020stable} and \textbf{SRDO}~\citep{shen2020stable}) with filter methods (including mutual information based (\textbf{MI}) and correlation based (\textbf{Correlation}) methods), wrapper methods (including gradient boosting (\textbf{GB})~\citep{friedman2001greedy}, XGBoost (\textbf{XGB})~\citep{chen2016xgboost}, and random forests (\textbf{RF})~\citep{diaz2006gene}), and embedded methods (including \textbf{OLS}, \textbf{LASSO}~\citep{tibshirani1996regression}, and \textbf{STG}~\citep{yamada2020feature}). More details on baseline implementations can be found in \appendixref{sect:experimental-detail}.

\vspace{-10px}

\paragraph{Evaluation metrics} On the one hand, to test the performances on feature selection, we report the rank average and F1 score of selected features. 
To compute the rank average, we utilize the scores that each model assigns to the features and then rank all features according to the scores. 
The rank average is calculated as the mean of the ranks of the minimal stable variable set $\pmbd(Y) = \boldS$. The F1 score is defined as the harmonic mean of the precision and recall, where precision and recall are computed by comparing the selected features to the true features, \textit{i.e.}, the minimal stable variable set. On the other hand, to test the performances on covariate-shift generalization, we calculate the root mean squared error (RMSE) in each test environment and report the mean and standard deviation of RMSE in various test environments.

\subsection{Experimental Results and Analysis} \sectionlabel{sect:exp-result}
The results are shown in \figureref{fig:experiment} and \figureref{fig:ood-S-2.5} and we have the following observations.
\begin{enumerate}[wide,noitemsep,topsep=0pt,parsep=0pt,partopsep=0pt]
    \item We first validate the optimality property of the minimal stable variable set on covariate-shift generalization proposed in \sectionref{sect:local-causal}.
    With fixed $r_{\text{tr}}=2.5$ and predicted feature ranking by SRDO, we vary the number of top selected features and train an MLP on them. Afterward, we test the performances of the MLP on test distributions and show the results in \figureref{fig:ood-S-2.5}. The results demonstrate that the minimal stable variable set (5 features) achieves the optimal performance under covariate-shift generalization. The figures with different $r_{\text{tr}}$ and more experimental details are provided in \appendixref{sect:experimental-detail}.
    \item Independence-driven IW algorithms perform much better on the discovery of the minimal stable variable set than other feature selection methods. As shown in \figureref{fig:experiment}, SRDO and DWR achieve the minimal rank average and maximal F1 score for both data-generating processes $Y_{\text{MLP}}$ and $Y_{\text{poly}}$. As a result, with the accurate discovery of the minimal stable variable set, SRDO and DWR further achieve the best covariate-shift generalization metrics (RMSE average and standard deviation). This experiment result validates the theories in \sectionref{sect:theory}.
    \item The discovery of the minimal stable variable set becomes progressively challenging as spurious correlation strength $r_{\text{tr}}$ increases. As shown in \figureref{fig:experiment}, the rank average tends to increase while the F1 score tends to decrease for all methods as $r_{\text{tr}}$ increases. This phenomenon makes sense on the grounds that $V_1$ and $V_2$ become strongly correlated with $Y$ and models tend to select them when $r_{\text{tr}}$ is large.
    \item SRDO outperforms DWR in most settings, especially when $Y$ has a complex non-linear relationship with $\boldX$. As discussed in \sectionref{sect:specific}, DWR aims to decorrelate the linear relationships between features and can not guarantee strict statistical independence. In the $Y_{\text{MLP}}$ setting of our experiment, DWR fails to discover the minimal stable variable set when $r_{\text{tr}}$ is large while SRDO performs much better in the setting. However, DWR could effectively handle the polynomial $Y_{\text{poly}}$ setting, which is also suggested by \citet{kuang2020stable}.
\end{enumerate}

%% file: paragraphs/discussions.tex
\section{Discussions} \sectionlabel{sect:discussions}
In this paper, we theoretically prove the effectiveness of independence-driven IW algorithms. We show that under ideal conditions, \textit{i.e.}, perfectly learned sample weights and infinite samples, the algorithms could identify the minimal stable variable set, which is the minimal set of variables that could provide good predictions under covariate shift. We further provide non-asymptotic properties and error analysis when these two conditions are not satisfied. Empirical results also demonstrate the superiority of these methods in selecting target variables.

\paragraph{Relationships between the minimal stable variable set and the Markov boundary}
The minimal stable variable set has close relationships with the Markov boundary, which we will further demonstrate in \appendixref{sect:markov-boundary}. Here we provide a brief discussion.

Firstly, we can easily verify that the minimal stable variable set is a subset of the Markov boundary (\theoremref{thrm:subset} and \exampleref{example:proper-subset}) by definition (\definitionref{defn:minimal-stable-set} and \definitionref{defn:markov-boundary}). However, not all variables in the Markov boundary are necessary for the covariate-shift generalization problem with common loss functions while the minimal stable variable set could provide the minimal set of variables (comparing \theoremref{thrm:ood-prediction} and \theoremref{thrm:markov-ood}).

In addition, traditional Markov boundary discovery algorithms mainly adopt the conditional independence test \citep{fukumizu2007kernel,sejdinovic2013equivalence,strobl2019approximate}, which is a particularly challenging hypothesis to test for~\citep{shah2020hardness} though. As a result, independence-driven IW algorithms would hopefully provide a proper approximation of the Markov boundary, which could be of independent interest.

\paragraph{Applicable scenarios and limitations}
We should notice that the definition of the minimal stable variable set is applicable only when $\mathbb{E}[Y|\boldX]$ is well defined. 
This implies that the definitions could be applied to typical regression and binary classification settings, but they may not be applicable in multi-class classification settings. 
In addition, under regression settings, $\mathbb{E}[Y | \boldX]$ will not be the solution in other forms of losses. 
For example, consider the Minkowski loss \citep[Section 1.5.5]{bishop2006pattern} given as $L_q = \mathbb{E}[|Y - f(\boldX)|^q]$.
It reduces to the expected squared loss when $q=2$. 
The minimum of $L_q$ is given by the conditional mean $\mathbb{E}[Y | \boldX]$ for $q = 2$, which is our case. 
But the solution becomes the conditional median for $q=1$ and the conditional mode for $q \rightarrow 0$.
Nevertheless, we highlight that the squared loss under regression settings and the cross-entropy loss under binary classification settings are general enough for most potential applications. 
We leave the theoretical analysis and applications of independence-driven IW algorithms on multi-class classification settings as future work.

%% file: paragraphs/acknowledgments.tex
\section*{Acknowledgments}
This work was supported in part by National Key R\&D Program of China (No. 2018AAA0102004), National Natural Science Foundation of China (No. 62141607, U1936219), and Beijing Academy of Artificial Intelligence (BAAI).

%% file: paragraphs/relationship.tex
\section{Relationships between the Minimal Stable Variable Set and the Markov Boundary} \sectionlabel{sect:markov-boundary}
\subsection{Main Results}
Generally speaking, the minimal stable variable set is closely related to the Markov boundary and independence-driven IW algorithms would hopefully provide a proper approximation of the Markov boundary. In addition, if setting covariate-shift generalization as the goal, the Markov boundary is not necessary while the minimal stable variable set is sufficient and optimal. Details are provided as follows.

\paragraph{Definition and basic property of the Markov blankets and boundary}
According to \citep{statnikov2013algorithms,pearl2014probabilistic}, Markov blankets and Markov boundary are defined as follows.

\begin{definition} [Markov blanket]
    A Markov blanket of $Y$ under distribution $P$ is any subset $\boldS$ of $\boldX$ for which
    \begin{equation} \equationlabel{eq:blanket}
        Y \perp (\boldX \backslash \boldS) \mid \boldS.
    \end{equation}
    The set of all Markov blankets for $Y$ is denoted as $\mbl_P(Y)$. In addition, we use $\mbl(Y)$ to denote the set under the training distribution $\ptrain$ for simplicity, \textit{i.e.}, $\mbl(Y) \triangleq \mbl_{\ptrain}(Y)$.
\end{definition}

\begin{definition} [Markov boundary] \definitionlabel{defn:markov-boundary}
    A Markov Boundary of $Y$ is a minimal Markov blanket of $Y$, \textit{i.e.}, none of its proper subsets satisfy \equationref{eq:blanket}.
\end{definition}

The existence of Markov blankets and Markov boundaries are given by the following proposition.

\begin{proposition} 
\propositionlabel{prop:unique-boundary-all-blankets}
    Under \assumptionref{assum:positive}, there exists a unique Markov boundary of $Y$, which can be denoted as $\mbd(Y)$.
    Furthermore, with the unique Markov boundary $\mbd(Y)$, the set of all Markov blankets of $Y$, $\mbl(Y)$, can be expressed as
    \begin{equation}
        \mbl(Y) = \{\boldS \subseteq \boldX \mid \mbd(Y) \subseteq \boldS\}.
    \end{equation}
\end{proposition}

\paragraph{Comparing the minimal stable variable set and the Markov boundary}
Besides the similarities in mathematical forms, there exist some connections between the stable variable set and the Markov blanket, and between the minimal stable variable set and the Markov boundary.

\begin{theorem} 
\theoremlabel{thrm:subset}
    Under \assumptionref{assum:positive}, a stable variable set is also a Markov blanket and the minimal stable variable set is a subset of the Markov boundary, \textit{i.e.},
    \begin{equation}
        \mbl(Y) \subseteq \pmbl(Y), \quad \pmbd(Y) \subseteq \mbd(Y).
    \end{equation}
\end{theorem}
The above theorem shows the inclusion relations between those two concepts, and the following example further illustrates a proper inclusion case.

\begin{example} [from \citet{strobl2016markov}] \examplelabel{example:proper-subset}
    Let $\boldX = (X_1, X_2)$ and the data-generating process is given as follows.
    \begin{equation}
        X_1, X_2 \sim N(0,1), \quad Y = f(X_1) + N\left(0, g(X_2)^2\right),
    \end{equation}
    where $f(\cdot)$ and $g(\cdot)$ are fixed functions. Then
    \begin{equation}
        \begin{aligned}
            & \{X_1\} = \pmbd(Y) \subsetneq \mbd(Y) = \{X_1, X_2\}, \\
            & \{\{X_1, X_2\}\} = \mbl(Y) \subsetneq \pmbl(Y) = \{\{X_1\}, \{X_1, X_2\}\}.
        \end{aligned}
    \end{equation}
\end{example}

The following proposition provides the property of the Markov boundary on covariate-shift generalization.

\begin{theorem} \theoremlabel{thrm:markov-ood}
    Under \assumptionref{assum:covariate-shift} and \assumptionref{assum:positive}, suppose $\mathbb{M}$ is a performance metric that is maximized only when $\ptest(Y | \boldX)$ is estimated accurately and $\mathbb{L}$ is a learning algorithm that can approximate any conditional probability distribution. Suppose $\boldS \subseteq \boldX$ is a subset of variables, then
    \begin{enum}
        \item $\boldS$ is an optimal predictor of $Y$ under the test distribution $\ptest$ if and only if it is a Markov blanket of $Y$ under the training distribution $\ptrain$, and
        \item $\boldS$ is a minimal and optimal predictor of $Y$ under the test distribution $\ptest$ if and only if it is a Markov boundary of $Y$ under the training distribution $\ptrain$.
    \end{enum}
\end{theorem}


\begin{remark}
The main difference between \theoremref{thrm:ood-prediction} and \theoremref{thrm:markov-ood} is the requirement on the performance metric $\mathbb{M}$. The Markov boundary is the minimal and optimal predictor if $\mathbb{M}$ is chosen as maximizing $\ptest(Y|\boldX)$.
However, for regression problems with the mean squared loss and binary classification problems with the cross-entropy loss, $\mathbb{E}_{\ptest}[Y|\boldX]$ is optimal in the test distribution $\ptest$.
\end{remark}

As a result, compared with the Markov boundary, the minimal stable variable set can bring two advantages.
\begin{enum}
    \item The conditional independence test is the crux to the precise discovery of the Markov boundary. \citet{shah2020hardness} have shown that conditional independence is a particularly challenging hypothesis to test for, which highlights the challenges of discovering the Markov boundary in real-world tasks. However, discovering the minimal stable variable set is relatively easier and proved possible in this paper.
    \item In several common machine learning tasks, including regression and binary classification, not all variables in the Markov boundary are necessary. As shown in \exampleref{example:proper-subset}, if a variable only affects the variance of the response variable $Y$, it would not be useful to predict $Y$ when adopting mean squared loss. The minimal stable variable set is proved to be a subset of the Markov boundary and it excludes useless variables in the Markov boundary for covariate-shift generalization.
\end{enum}

In addition, since the precise discovery of the Markov boundary is challenging, independence-driven IW algorithms would hopefully provide a proper approximation of the Markov boundary, which could be of independent interest.

\subsection{Related Works on Causal Discovery and Markov Boundary}
Causal literature can be categorized into two frameworks, namely the potential outcome \citep{rosenbaum1983central,holland1986statistics,rubin2005causal,imbens2015causal,johansson2016learning,zou2019focused,zou2020counterfactual} and the structural causal model framework \citep{pearl2014probabilistic}. The definition of the minimal stable variable set in this work is closely related to the Markov boundary, which falls into the structural causal model framework. Traditional causal discovery literature aims to discover the causal relationship between all variables. Typical methods include constraint-based \citep{spirtes2000causation,spirtes2013causal}, scored-based \citep{chickering2002optimal,huang2018generalized}, and learning-based \citep{zheng2018dags,zheng2020learning, he2021daring} methods.

Markov blankets and Markov boundary \citep{pearl2014probabilistic} are the cores of local causal discovery. Under the intersection assumption \citep{pearl2014probabilistic}, the Markov boundary is proved unique and the discovery algorithms include \citep{tsamardinos2003towards,tsamardinos2003time,tsamardinos2003algorithms,mani2004causal,aliferis2010locala,aliferis2010localb,pena2007towards}. Moreover, \citet{liu2010ensemblea,liu2010ensembleb,statnikov2013algorithms} studied the setting when multiple Markov boundaries exist. In this paper, we assume that the probabilities are strictly positive, which is a stronger assumption than the intersection assumption \citep{pearl2014probabilistic} but is also common in reality \citep{strobl2016markov}. With this assumption, we can guarantee the uniqueness of the Markov boundary and the minimal stable variable set proposed in this paper.

Traditional discovery algorithms mainly use the conditional independence test \citep{fukumizu2007kernel,sejdinovic2013equivalence,strobl2019approximate}. However, \citet{shah2020hardness} proved that conditional independence is indeed a particularly difficult hypothesis to test for and there is no free lunch in conditional independence testing, which limits the application of these methods in reality. \citep{strobl2016markov} propose a regression-based method to discover Markov boundaries, which is mostly closed to us. They proved that their method could find a subset of the Markov boundary but did not discuss the detailed properties of the subset. Here we further demonstrate that under ideal conditions, independence-driven IW algorithms could identify the exact subset of variables, \textit{i.e.}, the minimal stable variable set.

%% file: paragraphs/app-experiment.tex
\section{More Experimental Details} \sectionlabel{sect:experimental-detail}
\paragraph{Implementation details} We use scikit-learn\footnote{\url{https://scikit-learn.org}} for mutual information based (MI), correlation based (Correlation), gradient boosting (GB), random forest (RF), and LASSO, XGBoost package\footnote{\url{https://xgboost.readthedocs.io/en/stable/}} for XGBoost (XGB), and original implementation\footnote{\url{https://github.com/runopti/stg}} of STG. The hyperparameter search ranges for these baselines are shown in \tableref{tab:hyper}.

For independence-driven IW algorithms,  \textit{i.e.}, DWR, and SRDO, the feature scores are calculated as the absolute values of WLS coefficients. To be specific, for the DWR algorithm, following \citet{kuang2020stable}, we learn sample weights $\boldw = \{w_i\}_{i=1}^n$ by
\begin{equation} \equationlabel{eq:app-DWR}
    \boldw = \argmin_{\boldw \in \bbR^n} \sum_{1\le i,j \le d, i\ne j}\left(\widetilde{\cov}(X_i, X_j; \boldw)\right)^2 + \lambda_1 \cdot \left(\sum_{i=1}^n w_i - 1\right)^ 2 + \lambda_2 \cdot \sum_{i=1}^nw_i^2.
\end{equation}
Here $\widetilde{\cov}(X_i, X_j; \boldw)$ denotes the empirical covariance of $X_i$ and $X_j$ with sample weights $\boldw$. \equationref{eq:app-DWR} is optimized by the Adam algorithm~\citep{kingma2014adam} with a learning rate of $0.001$. For the SRDO algorithm, we implement it according to the official code\footnote{\url{https://github.com/Silver-Shen/Stable_Linear_Model_Learning}}. In detail, an MLP classifier (two hidden layers with sizes $30$ and $10$, respectively) is utilized to discriminate between the training distribution $\ptrain(\boldX)$ and the weighted distribution $\ptrain(X_1)\ptrain(X_2)\dots\ptrain(X_d)$. We adopt the binary cross-entropy loss and the Adam algorithm with a learning rate of $0.001$. To further restrict the range of learned sample weights, we clip the weights to $[1 / \gamma, \gamma]$. The search ranges of the hyperparameters are shown in \tableref{tab:hyper}.

We run each model $5$ times on various training datasets. In each run, we train feature selection models and get the top $5$ features selected by the algorithms. We then train an MLP regressor on these features to predict $Y$. The MLP adopted here has two hidden layers with sizes $5$ and $5$, respectively. The optimizer is the Adam algorithm with a learning rate of $0.001$. All MLPs in our paper adopt the ReLU activation function.

\begin{table}[h]
    \centering
    \caption{Hyperparameter search ranges of the methods.}
    \begin{tabular}{c|c}
        \toprule
        Method & Hyperparameters \\
        \midrule
        Mutual information based (MI) & $\mathrm{n\_neighbors} \in \{3, 5, 10, 20\}$ \\
        Correlation based (Correlation) & N/A \\
        Gradient boosting (GB) & $\mathrm{n\_estimators} \in \{50, 100, 200\}$, $\mathrm{max\_depth} \in \{6, 8 ,10\}$ \\
        XGBoost (XGB) & $\mathrm{n\_estimators} \in \{50, 100, 200\}$, $\mathrm{max\_depth} \in \{6, 8 ,10\}$ \\
        Random forests (RF) & $\mathrm{n\_estimators} \in \{50, 100, 200\}$, $\mathrm{max\_depth} \in \{6, 8 ,10\}$ \\
        OLS & N/A \\
        LASSO & $\alpha \in \{0.0003, 0.001, 0.01, 0.1\}$ \\
        STG & $\lambda \in \{0.001, 0.01, 0.1, 1.0, 10.0\}$ \\
        DWR & $\lambda_1 \in \{0.02, 0.05, 0.1\}$, $\lambda_2 \in \{0.02, 0.05, 0.1\}$ \\
        SRDO & $\gamma \in \{5, 10, 20\}$ \\
        \bottomrule
    \end{tabular}
    \tablelabel{tab:hyper}
\end{table}

\begin{figure}[h]
    \centering
    \begin{subfigure}[b]{0.33\linewidth}
        \centering
        \includegraphics[width=\linewidth]{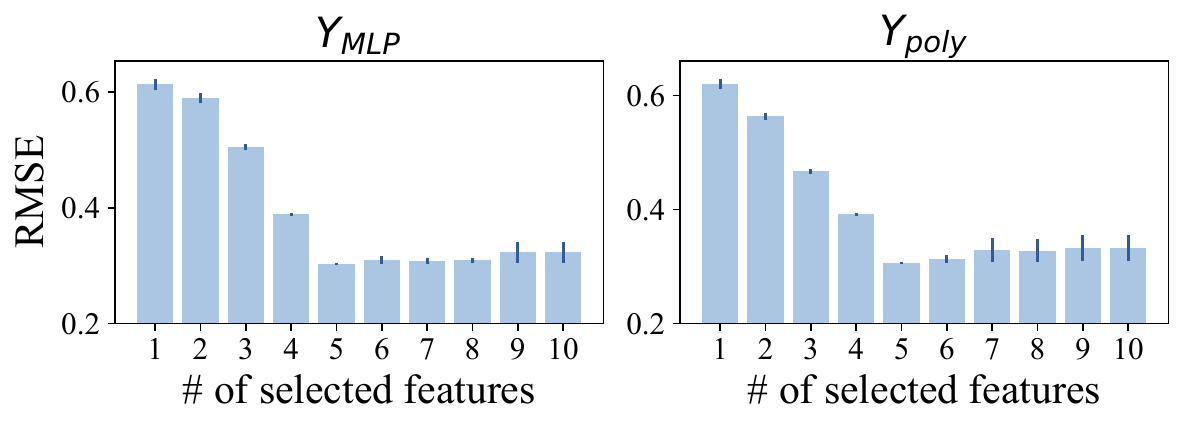}
        \caption{$r_{\text{tr}}=1.5$}
        \figurelabel{fig:ood-S-1.5}
    \end{subfigure}
    \hfill
    \begin{subfigure}[b]{0.33\linewidth}
        \centering
        \includegraphics[width=\linewidth]{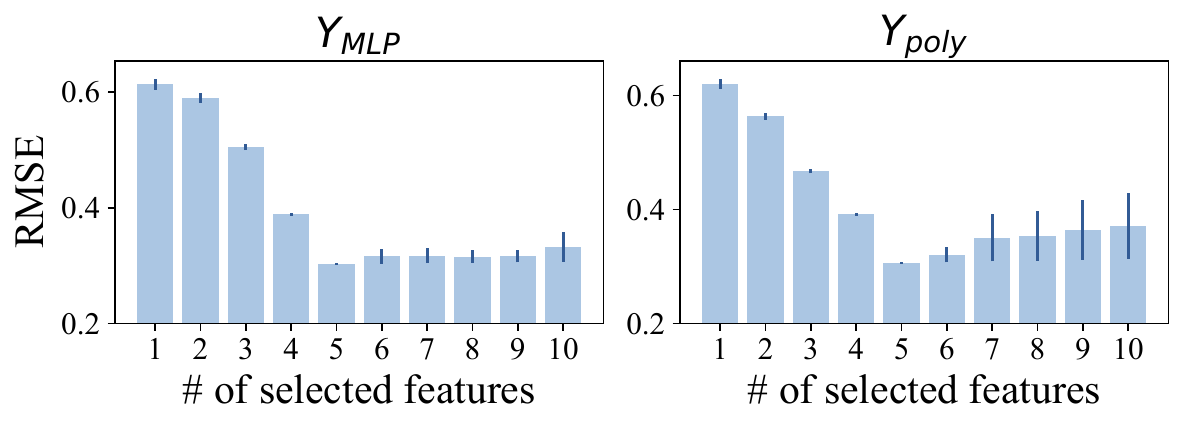}
        \caption{$r_{\text{tr}}=2.0$}
        \figurelabel{fig:ood-S-2.0}
    \end{subfigure}
    \hfill
    \begin{subfigure}[b]{0.33\linewidth}
        \centering
        \includegraphics[width=\linewidth]{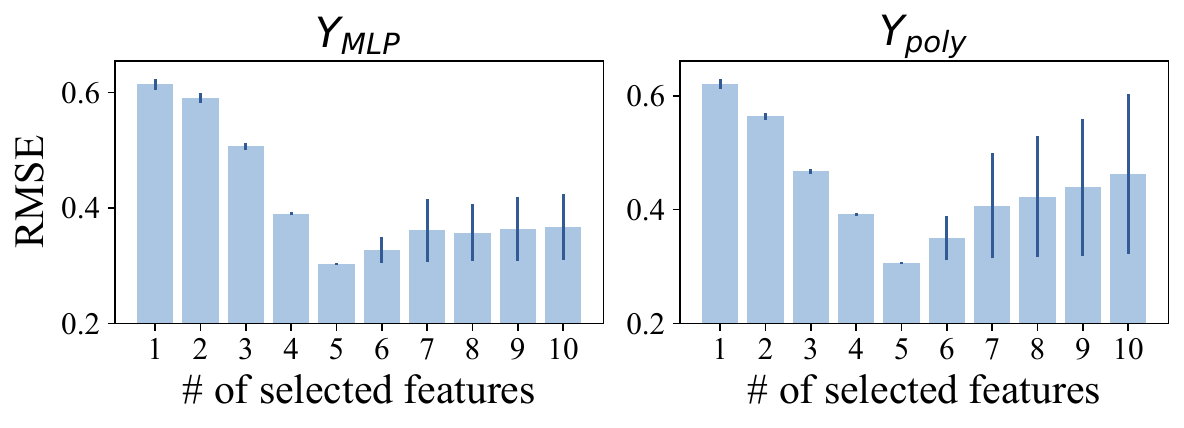}
        \caption{$r_{\text{tr}}=3.0$}
        \figurelabel{fig:ood-S-3.0}
    \end{subfigure}
    \caption{The covariate-shift generalization metrics (RMSE average and standard deviation) \textit{w.r.t.} the number of selected features. Fix $r_{\text{tr}}=1.5, 2.0, 3.0$ and the feature ranking lists are provided by SRDO. The minimal stable variable set (5 features) achieves the optimal performance.}
    \figurelabel{fig:ood-S-more}
\end{figure}

\begin{table}[h]
    \centering
    \caption{Feature rankings in \figureref{fig:ood-S-2.5} and \figureref{fig:ood-S-more}.}
    \begin{tabular}{c|c|c}
        \toprule
        $r_{\text{train}}$ & $Y_{\text{MLP}}$ & $Y_{\text{poly}}$ \\
        \midrule
        $1.5$ & $S_3, S_5, S_2, S_1, S_4, V_5, V_4, V_2, V_1, V_3$ & $S_3, S_2, S_5, S_1, S_4, V_4, V_3, V_1, V_5, V_2$ \\
        $2.0$ & $S_3, S_5, S_2, S_1, S_4, V_5, V_4, V_3, V_2, V_1$ & $S_3, S_2, S_5, S_1, S_4, V_4, V_3, V_1, V_2, V_5$ \\
        $2.5$ & $S_3, S_2, S_5, S_1, S_4, V_5, V_4, V_2, V_3, V_1$ & $S_3, S_2, S_5, S_1, S_4, V_4, V_5, V_2, V_3, V_1$ \\
        $3.0$ & $S_3, S_5, S_2, S_1, S_4, V_4, V_5, V_3, V_1, V_2$ & $S_3, S_2, S_5, S_1, S_4, V_5, V_4, V_2, V_1, V_3$ \\
        \bottomrule
    \end{tabular}
    \tablelabel{tab:feature-ranking}
\end{table}

\paragraph{The optimality property of the minimal stable variable set} We adopt the SRDO method to generate feature rankings in both the $Y_{\text{MLP}}$ (We only sample one MLP in this experiment.) and $Y_{\text{poly}}$ settings. The feature rankings in different settings are shown in \tableref{tab:feature-ranking}.

We vary $r_{\text{tr}}$ to test the covariate-shift generalization metrics \textit{w.r.t.} the number of selected features. The results are shown in \figureref{fig:ood-S-more}. We observe a similar phenomenon as that shown in \figureref{fig:ood-S-2.5}. These empirical results validate the advantage of the minimal stable variable set on covariate-shift generalization. 

%% file: paragraphs/proofs.tex
\section{Omitted Proofs}
\subsection{Proof of \texorpdfstring{\theoremref{thrm:ood-prediction}}{Theorem}}
\begin{lemma} \lemmalabel{lemma:stable-sets-and-iid}
    Under \assumptionref{assum:positive}, suppose $\mathbb{M}$ is a performance metric that is maximized only when $\mathbb{E}_{\ptrain}[Y | \boldX]$ is estimated accurately and $\mathbb{L}$ is a learning algorithm that can approximate any conditional expectation.  Suppose $\boldS \subseteq \boldX$ is a subset of variables, then
    \begin{enum}
        \item $\boldS$ is an optimal predictor of $Y$ if and only if it is a stable variable set of $Y$ under distribution $\ptrain$, and
        \item $\boldS$ is a minimal and optimal predictor of $Y$ if and only if it is a minimal stable variable set of $Y$ under distribution $\ptrain$.
    \end{enum}
\end{lemma}

\begin{proof}[Proof of \lemmaref{lemma:stable-sets-and-iid}]
    We omit the subscript of $\mathbb{E}_{\ptrain}[\cdot | \cdot]$ for simplicity.

    Consider the first part. On the one hand, if $\boldS$ is a stable variable set of $Y$, then $\mathbb{E}[Y |\boldX] = \mathbb{E}[Y | \boldS]$ by definition. Hence $\boldS$ is an optimal predictor because $\mathbb{E}[Y |\boldX] = \mathbb{E}[Y | \boldS]$ can be approximated perfectly by $\mathbb{L}$ and $\mathbb{M}$ will be maximized. On the other hand, assume $\boldS$ is an optimal predictor but not a stable variable set, which implies that $\mathbb{E}[Y|\boldS]\ne\mathbb{E}[Y|\boldX]$. $\boldX$ is a stable variable set by definition. Hence, By first part of the proof, $\boldX$ is an optimal predictor of $Y$, similar to $\boldS$. Therefore, the following should hold: $\mathbb{E}[Y|\boldX] = \mathbb{E}[Y | \boldS]$, which contradicts the assumption that $\boldS$ is not a stable variable set. As a result, $\boldS$ is a stable variable set of $Y$.
    
    Consider the second part. On the one hand, if $\boldS$ is a minimal stable variable set of $Y$, then it is also a stable variable set of $Y$. So $\boldS$ is an optimal predictor. Moreover, by the definition of the minimal stable variable set, no proper subset of $\boldS$ is a stable variable set of $Y$. Therefore, no proper subset of $\boldS$ satisfies the definition of an optimal predictor. Thus, $\boldS$ is a minimal and optimal predictor of $Y$. On the other hand, assume $\boldS$ is a minimal and optimal predictor of $Y$. Then, $\boldS$ is also an optimal predictor of Y, which implies that $\boldS$ is a stable variable set of $Y$. By the definition of minimality, no proper subset of $\boldS$ is a minimal and optimal predictor. Hence, no proper subset of $\boldS$ is a stable variable set of $Y$. As a result, $\boldS$ is a minimal stable variable of $Y$.
\end{proof}

Now we prove the original theorem.
\begin{proof}[Proof of \theoremref{thrm:ood-prediction}]
    It is obvious that $\mathbb{E}_{\ptrain}[Y|\boldX] = \mathbb{E}_{\ptest}[Y|\boldX]$ from \assumptionref{assum:covariate-shift}. As a result, the original theorem is proved according to \lemmaref{lemma:stable-sets-and-iid}.
\end{proof}

\subsection{Proof of \texorpdfstring{\theoremref{thrm:unique-stable-set}}{Theorem}}
The proof is based on the following intersection property.
\begin{lemma} \lemmalabel{lemma:partial-intersection}
    Under \assumptionref{assum:positive}, if $\boldS_1, \boldS_2 \in \pmbl(Y)$, then $\boldS_1 \cap \boldS_2 \in \pmbl(Y)$.
\end{lemma}
\begin{proof}[Proof of \lemmaref{lemma:partial-intersection}]
    Let $\boldS = \boldS_1 \cap \boldS_2$, $\bar{\boldS}_1 = \boldS_1 \backslash \boldS$, $\bar{\boldS}_2 = \boldS_2 \backslash \boldS$, and $\bar{\boldX} = \boldX \backslash (\boldS_1 \cup \boldS_2)$.
    Then $\boldX = (\boldS, \bar{\boldS}_1, \bar{\boldS}_2, \bar{\boldX})$.
    
    By definition, $\forall \bolds \in \mathcal{S}, \bar{\bolds}_1 \in \bar{\mathcal{S}}_1, \bar{\bolds}_2 \in \bar{\mathcal{S}}_2, \bar{\boldx} \in \bar{\mathcal{X}}$, $\mathbb{E}[Y|\boldS=\bolds, \bar{\boldS}_1 = \bar{\bolds}_1] = \mathbb{E}[Y|\boldS=\bolds, \bar{\boldS}_2 = \bar{\bolds}_2] = \mathbb{E}[Y|\boldS=\bolds, \bar{\boldS}_1 = \bar{\bolds}_1, \bar{\boldS}_2 = \bar{\bolds}_2, \bar{\boldX}=\bar{\boldx}]$. Let $\boldx = (\bolds, \bar{\bolds}_1, \bar{\bolds}_2, \bar{\boldx})$. Under  \assumptionref{assum:positive},
    \begin{equation*}
        \begin{aligned}
            & \mathbb{E}[Y|\boldS=\bolds] \\
            = & \int_{\mathcal{Y}}y \ptrain(Y=y|\boldS=\bolds)\mathrm{d}y \\
            = & \int_\mathcal{Y}\int_{\bar{\mathcal{S}}_1} y \ptrain(Y=y|\boldS=\bolds, \bar{\boldS}_1 = \bar{\bolds}_1)\ptrain(\bar{\boldS}_1 = \bar{\bolds}_1 | \boldS=\bolds)\mathrm{d}\bar{\bolds}_1\mathrm{d}y \\
            = & \int_{\bar{\mathcal{S}}_1}\mathbb{E}[Y|\boldS=\bolds, \bar{\boldS}_1 = \bar{\bolds}_1]\ptrain(\bar{\boldS}_1 = \bar{\bolds}_1 | \boldS=\bolds)\mathrm{d}\bar{\bolds}_1 \\
            = & \mathbb{E}[Y|\boldS=\bolds, \bar{\boldS}_2 = \bar{\bolds}_2] \int_{\bar{\mathcal{S}}_1}\ptrain(\bar{\boldS}_1 = \bar{\bolds}_1 | \boldS=\bolds)\mathrm{d}\bar{\bolds}_1 \\
            = & \mathbb{E}[Y|\boldS=\bolds, \bar{\boldS}_2 = \bar{\bolds}_2] = \mathbb{E}[Y|\boldS=\bolds, \bar{\boldS}_1 = \bar{\bolds}_1, \bar{\boldS}_2 = \bar{\bolds}_2, \bar{\boldX}=\bar{\boldx}] = \mathbb{E}[Y | \boldX=\boldx].
        \end{aligned}
    \end{equation*}
    As a result, $\boldS \in \pmbl(Y)$.
\end{proof}
Now we prove the original theorem.
\begin{proof}[Proof of \theoremref{thrm:unique-stable-set}]
    We first prove the uniqueness of the minimal stable variable set. Suppose there are two minimal stable variable sets \textit{w.r.t.} $Y$, denoted as $\pmbd_1(Y)$ and $\pmbd_2(Y)$. By definition, $\pmbd_1(Y), \pmbd_2(Y) \in \pmbl(Y)$. Under \assumptionref{assum:positive}, according to \lemmaref{lemma:partial-intersection}, $\pmbd_1(Y) \cap \pmbd_2(Y) \in \pmbl(Y)$. Because $\pmbd_1(Y)$ has no proper subset that is in $\pmbl(Y)$, we have $\pmbd_1(Y) \cap \pmbd_2(Y) = \pmbd_1(Y)$. Similarly, $\pmbd_1(Y) \cap \pmbd_2(Y) = \pmbd_2(Y)$, which means $\pmbd_1(Y) = \pmbd_2(Y)$.

    Next, we prove the exact form of the stable variable sets. Let
    $$
    \Omega = \{\boldS \subseteq \boldX \mid \pmbd(Y) \subseteq \boldS\}.
    $$
    On the one hand, $\forall \boldS \in \pmbl(Y)$, according to \lemmaref{lemma:partial-intersection}, $\boldS \cap \pmbd(Y) \in \pmbl(Y)$. Because of the minimality of $\pmbd(Y)$, $|\boldS \cap \pmbd(Y)| \ge |\pmbd(Y)|$. As a result, $\pmbd(Y)\subseteq \boldS$ and $\boldS \in \Omega$. Hence $\pmbl(Y) \subseteq \Omega$.
    
    On the other hand, $\forall \boldS \in \Omega$, let $\boldD = \pmbd(Y)$, $\boldW = \boldS \backslash \boldD$, and $\bar{\boldX} = \boldX \backslash \boldS$. Then $\forall \boldd \in \calD, \boldw \in \calW$, $\bolds = (\boldd, \boldw)$, we can get
    $$
    \begin{aligned}
        & \mathbb{E}[Y|\boldS = \bolds] = \int_{\calY}y \ptrain(Y=y|\boldD=\boldd, \boldW=\boldw)\mathrm{d}y \\
        = & \int_{\calY} \int_{\bar{\calX}}y \ptrain(Y=y|\boldD=\boldd, \boldW=\boldw, \bar{\boldX}=\bar{\boldx}) \ptrain(\bar{\boldX}=\bar{\boldx}|\boldD=\boldd, \boldW=\boldw)\mathrm{d}\bar{\boldx}\mathrm{d}y \\
        = & \int_{\bar{\calX}}\ptrain(\bar{\boldX}=\bar{\boldx}|\boldD=\boldd, \boldW=\boldw)\mathbb{E}[Y|\boldD=\boldd, \boldW=\boldw, \bar{\boldX}=\bar{\boldx}]\mathrm{d}\bar{\boldx} \\
        = & \int_{\bar{\calX}}\ptrain(\bar{\boldX}=\bar{\boldx}|\boldD=\boldd, \boldW=\boldw)\mathbb{E}[Y|\boldD=\boldd]\mathrm{d}\bar{\boldx} \\
        = & \mathbb{E}[Y|\boldD=\boldd]\int_{\bar{\calX}}\ptrain(\bar{\boldX}=\bar{\boldx}|\boldD=\boldd, \boldW=\boldw)\mathrm{d}\bar{\boldx} \\
        = & \mathbb{E}[Y|\boldD=\boldd] = \mathbb{E}[Y | \boldX=\boldx].
    \end{aligned}
    $$
    As a result, $\boldS$ satisfies the requirement of stable variable sets and $\boldS \in \pmbl(Y)$. Hence $\Omega \subseteq \pmbl(Y)$.
    
    To conclude, $\pmbl(Y) \subseteq \Omega$ and $\Omega \subseteq \pmbl(Y)$, which results in $\Omega = \pmbl(Y)$.
\end{proof}

\subsection{Proof of \texorpdfstring{\theoremref{thrm:findV}}{Theorem}}
We need the following lemma first.
\begin{lemma} \lemmalabel{lemma:invariance}
    Let $w \in \mathcal{W}$ be a weighting function, and $\tilde{P}_w$ be the corresponding weighted distribution. Then $\tilde{P}_w(Y |\boldX) = \ptrain(Y | \boldX)$.
\end{lemma}
\begin{proof}[Proof of  \lemmaref{lemma:invariance}]
    $\forall \boldx \in \calX, y \in \calY$,
    $$
    \begin{aligned}
        & \tilde{P}_w(Y = y | \boldX = \boldx) =  \frac{\tilde{P}_w(Y = y, \boldX=\boldx)}{\tilde{P}_w(\boldX=\boldx)} = \frac{\ptrain(Y = y, \boldX=\boldx)w(\boldx)}{\int_{y'}\tilde{P}_w(\boldX, Y=y')\mathrm{d} y'} \\
        = & \frac{\ptrain(Y = y, \boldX=\boldx)w(\boldx)}{w(\boldx)\int_{y'}\ptrain(\boldX=x, Y=y')\mathrm{d} y'} = \frac{\ptrain(Y=y, \boldX=\boldx)}{\ptrain(\boldX=\boldx)} = \ptrain(Y=y|\boldX=\boldx).
    \end{aligned}
    $$
\end{proof}

Now we prove the original theorem.
\begin{proof}[Proof of \theoremref{thrm:findV}]
    Let $\boldX_{-i}$ denote variables other than $X_i$ and $\mathcal{X}_{-i}$ denote the support of $\boldX_{-i}$.

    Given $X_i \not\in \pmbd(Y)$, there exists a function $f: \mathcal{X}_{-i} \rightarrow \mathcal{Y}$ such that $\mathbb{E}_{\ptrain(\boldX, Y)}[Y|\boldX] = f(\boldX_{-i})$. According to \lemmaref{lemma:invariance}, $\mathbb{E}_{\tilde{P}_w(\boldX, Y)}[Y|\boldX] = \mathbb{E}_{\ptrain(\boldX, Y)}[Y|\boldX] = f(\boldX_{-i})$. As a result, because $\mathbb{E}_{\ptrain(\boldX)}\left[w(\boldX)\lVert \boldX\rVert_2^2\right] < \infty$ and $\mathbb{E}_{\ptrain(\boldX, Y)}\left[w(\boldX)Y^2\right] < \infty$, the covariance between $X_i$ and $Y$ under $\tilde{P}_w$ is
    $$
    \begin{aligned}
        \cov_{\tilde{P}_w}[X_iY] =
        & \mathbb{E}_{\tilde{P}_w(X_i, Y)}[X_i Y] - \mathbb{E}_{\tilde{P}_w(X_i)}[X_i]\mathbb{E}_{\tilde{P}_w(Y)}[Y] \\
        = & \mathbb{E}_{\tilde{P}_w(\boldX)}\left[X_i\mathbb{E}_{\tilde{P}_w(\boldX, Y)}[Y|\boldX]\right]- \mathbb{E}_{\tilde{P}_w(X_i)}[X_i]\mathbb{E}_{\tilde{P}_w(\boldX)}\left[\mathbb{E}_{\tilde{P}_w(\boldX, Y)}[Y|\boldX]\right] \\
        = & \mathbb{E}_{\tilde{P}_w(\boldX)}[X_i f(\boldX_{-i})] - \mathbb{E}_{\tilde{P}_w(X_i)}[X_i]\mathbb{E}_{\tilde{P}_w(\boldX_{-i})}\left[f\left(\boldX_{-i}\right)\right] = 0.
    \end{aligned}
    $$
    The last equation is due to the independence between $X_i$ and $\boldX_{-i}$ in the weighted distribution $\tilde{P}_w$. As a result, the coefficient $\boldbeta_w(X_i)$ is
    $$
        \boldbeta_w(X_i) = \var_{\tilde{P}_w}(X_i)^{-1}\cov_{\tilde{P}_w}[X_iY] = 0.
    $$
\end{proof}

\subsection{Proof of \texorpdfstring{\theoremref{thrm:findS}}{Theorem}}
\begin{proof}
    Let $\boldX_{-i}$ denote the rest variable except $X_i$ and $\ptrain_{-i}$ denote the marginal distribution of $\ptrain$ on $\boldX_{-i}$. Because $X_i \in \pmbd(Y)$, $\mathbb{E}_{\ptrain(\boldX, Y)}[Y | \boldX]$ depends on $X_i$. Hence, there exists a probability density function $\tilde{P}_{-i}$ with the same support of $\ptrain_{-i}$ that satisfies
    \begin{enum}
        \item $\boldX_{-i}$ are mutually independent under $\tilde{P}_{-i}$, and
        \item $g(X_i) \triangleq \mathbb{E}_{\tilde{P}_{-i}(\boldX_{-i})}[\mathbb{E}_{\ptrain(\boldX, Y)}\left[Y | \boldX_{-i}, X_i]\right]$ depends on $X_i$.
    \end{enum}
    Moreover, there exist a probability density function $\tilde{P}_i$ with the same support of $\ptrain_i$ that satisfies $g(X_i)$ is linearly correlated with $X_i$ under $\tilde{P}_i$.
    
    Let $\tilde{P}$ be the joint distribution on $(\boldX, Y)$ and $\tilde{P}(\boldX_{-i}, X_i, Y) = \tilde{P}_{-i}(\boldX_{-i})\tilde{P}_i(X_i)\ptrain(Y|X)$. Hence,
    $$
    \mathbb{E}_{\tilde{P}(X_i, Y)}[Y | X_i] = \mathbb{E}_{\tilde{P}_{-i}(\boldX_{-i})}[\mathbb{E}_{\ptrain(\boldX, Y)}\left[Y | \boldX_{-i}, X_i]\right] = g(X_i).
    $$
    Let $w(\boldX) = \tilde{P}(\boldX) / \ptrain(\boldX)$. Because $\mathbb{E}_{\ptrain(\boldX, Y)}[Y | \boldX]$ depends on $X_i$, $\var_{\ptrain}(X_i) > 0$. Hence, $\var_{\tilde{P}}(X_i) > 0$. As a result, the coefficient on $X_i$ is
    $$
    \begin{aligned}
        & \boldbeta_w(X_i) \\
        = & \frac{1}{\var_{\tilde{P}_i}(X_i)}\left(\mathbb{E}_{\ptrain(\boldX, Y)}[w(\boldX)X_iY] - \mathbb{E}_{\ptrain(\boldX)}[w(\boldX)X_i]\mathbb{E}_{\ptrain(\boldX, Y)}[w(\boldX)Y]\right) \\
        = & \frac{1}{\var_{\tilde{P}_i}(X_i)}\left(\mathbb{E}_{\tilde{P}(X_i, Y)}[X_iY] - \mathbb{E}_{\tilde{P}(X_i)}[X_i]\mathbb{E}_{\tilde{P}(Y)}[Y]\right) \\
        = & \frac{1}{\var_{\tilde{P}_i}(X_i)}\left(\mathbb{E}_{\tilde{P}(X_i)}\left[X_i\mathbb{E}_{\tilde{P}(X_i, Y)}[Y|X_i]\right] - \mathbb{E}_{\tilde{P}_i(X_i)}[X_i] \mathbb{E}_{\tilde{P}_i(X_i)}\left[\mathbb{E}_{\tilde{P}(X_i, Y)}[Y | X_i]\right]\right) \\
        = & \frac{1}{\var_{\tilde{P}_i}(X_i)}\left(\mathbb{E}_{\tilde{P}_i(X_i)}[X_i g(X_i)] - \mathbb{E}_{\tilde{P}_i(X_i)}[X_i]\mathbb{E}_{\tilde{P}_i(X_i)}[g(X_i)]\right) \ne 0.
    \end{aligned}
    $$
\end{proof}

\subsection{Proof of \texorpdfstring{\theoremref{thrm:big-asumptotic}}{Theorem}}
We observe that
$$
\begin{aligned}
    \left\|\hat{\boldbeta}_{\hat{w}} - \boldbeta_{w}\right\|_2^2 & \le \left(\left\|\hat{\boldbeta}_{\hat{w}} - \boldbeta_{\hat{w}}\right\|_2 + \left\|\boldbeta_{\hat{w}} - \boldbeta_{w}\right\|_2\right)^2\\
    & \le 2\left(\left\|\hat{\boldbeta}_{\hat{w}} - \boldbeta_{\hat{w}}\right\|_2^2 + \left\|\boldbeta_{\hat{w}} - \boldbeta_{w}\right\|_2^2\right) \\
    & = 2\left(\left\|\Sigma_{\hat{w}}^{-1/2}\Sigma_{\hat{w}}^{1/2}\left(\hat{\boldbeta}_{\hat{w}} - \boldbeta_{\hat{w}}\right)\right\|_2^2 + \left\|\boldbeta_{\hat{w}} - \boldbeta_{w}\right\|_2^2\right) \\
    & \le 2\left(\left\|\Sigma_{\hat{w}}^{-1/2}\right\|_2^2\left\|\Sigma_{\hat{w}}^{1/2}\left(\hat{\boldbeta}_{\hat{w}} - \boldbeta_{\hat{w}}\right)\right\|_2^2 + \left\|\boldbeta_{\hat{w}} - \boldbeta_{w}\right\|_2^2\right) \\
    & = 2\left(\left\|\Sigma_{\hat{w}}^{-1}\right\|_2\left\|\hat{\boldbeta}_{\hat{w}} - \boldbeta_{\hat{w}}\right\|_{\Sigma_{\hat{w}}}^2 + \left\|\boldbeta_{\hat{w}} - \boldbeta_{w}\right\|_2^2\right)
\end{aligned}
$$
We analyze the upper bounds of the terms in the above equation and the first part of the claim follows from \propositionref{prop:spectral-lower-bound}, \propositionref{prop:condition-satisfied}, and \propositionref{prop:error-weightes}. Furthermore, the second part of the claim is then straightforward from \theoremref{thrm:findV} and \theoremref{thrm:findS}.

\subsubsection{Error Caused by WLS from Finite Samples}
\begin{proposition} \propositionlabel{prop:spectral-lower-bound}
    Suppose \assumptionref{assum:target-weighting} (with parameter $\boundlambda$) and \assumptionref{assum:estimate-weighting} (with parameter $\epsilon$) hold. Then $\left\|\Sigma_{\hat{w}}^{-1}\right\|_2 \le 1/\left(\boundlambda - \epsilon \sqrt{\mathbb{E}\left[\|\boldX\|_2^4\right]}\right)$.
\end{proposition}
\begin{proof}
    Let $\Delta_w(\boldX) \triangleq \hat{w}(\boldX) - w(\boldX)$ and $\Delta_{\Sigma}  = \mathbb{E}\left[\Delta_w(\boldX)\boldX\boldX^T\right]$.
    $$
    \begin{aligned}
        & \|\Delta_{\Sigma}\|_2 \\
        = & \sup_{\|c\|_2=1} \|\Delta_{\Sigma} \cdot c\|_2 = \sup_{\|c\|_2=1} \left\|\bbE\left[\Delta_w(\boldX)\boldX\boldX^Tc\right]\right\|_2 \\
        \le & \sup_{\|c\|_2=1} \mathbb{E}\left[\Delta_w(\boldX) \|\boldX\boldX^Tc\|_2\right]  && \text{(triangle inequality of norms)} \\
        \le & \sup_{\|c\|_2=1} \sqrt{\mathbb{E}\left[\Delta_w(\boldX)^2\right] \mathbb{E}\left[\|\boldX\boldX^Tc\|_2^2\right]} && \text{(Cauchy–Schwarz inequality)} \\
        = & \epsilon \sup_{\|c\|_2=1} \sqrt{\mathbb{E}\left[\|\boldX\boldX^Tc\|_2^2\right]} && (\mathbb{E}\left[\Delta_w(\boldX)^2\right] = \epsilon) \\
        \le & \epsilon \sqrt{\mathbb{E}\left[\sup_{\|c\|_2=1}\|\boldX\boldX^Tc\|_2^2\right]} && (\sup \mathbb{E}[\cdot] \le \mathbb{E}[\sup\cdot]) \\
        = & \epsilon \sqrt{\mathbb{E}\left[\|\boldX\|_2^4\right]}
    \end{aligned}
    $$
    As a result, according to Weyl's theorem \citep{horn2012matrix}, \assumptionref{assum:target-weighting}, and \assumptionref{assum:estimate-weighting},
    $$
    \minlambda\left(\Sigma_{\hat{w}}\right) = \minlambda\left(\Sigma_w + \Delta_{\Sigma} \right) \ge \boundlambda - \|\Delta_{\Sigma}\|_2 \ge \boundlambda - \epsilon \sqrt{\mathbb{E}\left[\|\boldX\|_2^4\right]} > 0.
    $$
    Therefore,
    $$
    \left\|\Sigma_{\hat{w}}^{-1}\right\|_2 = \frac{1}{\minlambda\left(\Sigma_{\hat{w}}\right)} \le \frac{1}{\boundlambda - \epsilon \sqrt{\mathbb{E}\left[\|\boldX\|_2^4\right]}}
    $$
\end{proof}

\begin{proposition} \propositionlabel{prop:condition-satisfied}
    Suppose \assumptionref{assum:bound-covariate} (with parameter $\boundx$), \assumptionref{assum:bound-appr} (with parameter $\boundappr$), \assumptionref{assum:noise} (with parameter $\sigma$), \assumptionref{assum:target-weighting} (with parameter $\boundlambda$), \assumptionref{assum:estimate-weighting} (with parameter $\epsilon$), and \assumptionref{assum:bound-weight} (with parameter $\boundweight$) hold. Then there exist constants $\rho_{\hat{w}}, b_{\hat{w}}, \sigma_{\hat{w}} > 0$ such that weighting function $\hat{w}$ satisfies \conditionref{cond:bounded-statistical} (with parameter $\rho_{\hat{w}}$), \conditionref{cond:bounded-approx} (with parameter $b_{\hat{w}}$), and \conditionref{cond:sub-gaussian} (with parameter $\sigma_{\hat{w}}$) and
    $$
    \left\{
    \begin{aligned}
        \rho_{\hat{w}} & \le \frac{\sqrt{\boundweight}\boundx}{\sqrt{d\left(\boundlambda - \epsilon \sqrt{\mathbb{E}\left[\|\boldX\|_2^4\right]}\right)}}, \\
        b_{\hat{w}} & \le \frac{\boundweight\boundappr\boundx}{\sqrt{d\left(\boundlambda - \epsilon \sqrt{\mathbb{E}\left[\|\boldX\|_2^4\right]}\right)}}, \\
        \sigma_{\hat{w}} & \le \sqrt{\boundweight}\sigma.
    \end{aligned}
    \right.
    $$
    Furthermore, pick any $t > \max\{0, 2.6 - \log d\}$, let
    $$
    n \ge \frac{6\boundweight \boundx^2 (\log d + t)}{\boundlambda - \epsilon \sqrt{\mathbb{E}\left[\|\boldX\|_2^4\right]}}.
    $$
    Then with probability at least $1 - 3e^{-t}$,
    $$
    \begin{aligned}
        \left\|\hat{\boldbeta}_{\hat{w}} - \boldbeta_{\hat{w}}\right\|_{\Sigma_{\hat{w}}}^2 & \le \frac{2\boundweight\sigma^2(d + 2\sqrt{td} + 2t)}{n} + \frac{4\boundweight\boundx^2\bbE\left[\hat{w}(\boldX)\appr(\boldX)^2\right]}{n\left(\boundlambda - \epsilon \sqrt{\mathbb{E}\left[\|\boldX\|_2^4\right]}\right)}\left(1+\sqrt{8t}\right)^2 + o(1 / n) \\
        & \le \frac{2\boundweight\sigma^2(d + 2\sqrt{td} + 2t)}{n} + \frac{4\boundweight\boundx^2\boundappr^2(1+\epsilon)}{n\left(\boundlambda - \epsilon \sqrt{\mathbb{E}\left[\|\boldX\|_2^4\right]}\right)}\left(1+\sqrt{8t}\right)^2 + o(1 / n).
    \end{aligned}
    $$
\end{proposition}
\begin{proof}
    Based on \assumptionref{assum:bound-covariate}, \assumptionref{assum:target-weighting}, \assumptionref{assum:estimate-weighting}, and \assumptionref{assum:bound-weight}, according to \propositionref{prop:spectral-lower-bound}, almost surely,
    $$
    \frac{\sqrt{\hat{w}(\boldX)}\left\|\Sigma_{\hat{w}}^{-1/2}\boldX\right\|_2}{\sqrt{d}} \le \frac{\sqrt{\boundweight}\left\|\Sigma_{\hat{w}}^{-1/2}\right\|_2\|\boldX\|_2}{\sqrt{d}} \le \frac{\sqrt{\boundweight}\boundx}{\sqrt{d\left(\boundlambda - \epsilon \sqrt{\mathbb{E}\left[\|\boldX\|_2^4\right]}\right)}}
    $$
    Based on \assumptionref{assum:bound-covariate}, \assumptionref{assum:bound-appr}, \assumptionref{assum:target-weighting}, \assumptionref{assum:estimate-weighting}, and \assumptionref{assum:bound-weight}, according to \propositionref{prop:spectral-lower-bound}, almost surely,
    $$
    \frac{\hat{w}(\boldX)\left\|\Sigma_{\hat{w}}^{-1/2}\appr(\boldX)\boldX\right\|_2}{\sqrt{d}} \le \frac{\boundweight\boundappr\boundx\left\|\Sigma_{\hat{w}}^{-1/2}\right\|}{\sqrt{d}} \le \frac{\boundweight\boundappr\boundx}{\sqrt{d\left(\boundlambda - \epsilon \sqrt{\mathbb{E}\left[\|\boldX\|_2^4\right]}\right)}}
    $$
    Based on \assumptionref{assum:noise} and \assumptionref{assum:bound-weight}, almost surely,
    $$
    \forall \eta \in \bbR, \quad \bbE\left[\left.\exp\left(\eta\sqrt{\hat{w}(\boldX)}\noise(\boldX)\right) \right| \boldX\right] \le \exp\left(\eta^2 \hat{w}(\boldX) \sigma^2 /2\right) \le \exp\left(\eta^2 \left(\sqrt{\boundweight}\sigma\right)^2/2\right).
    $$
    Because
    $$
    \begin{aligned}
        \bbE\left[\hat{w}(\boldX)\appr(\boldX)^2\right] & \le \boundappr^2\bbE\left[|\hat{w}(\boldX)|\right] \le \boundappr^2\left(\bbE[w(\boldX)] + \bbE[|w(\boldX) - \hat{w}(\boldX)|]\right) \\
        & \le \boundappr^2\left(1 + \sqrt{\bbE\left[\left(w(\boldX) - \hat{w}(\boldX)\right)^2\right]}\right) = \boundappr^2(1 + \epsilon).
    \end{aligned}
    $$
    Now the claim follows from \theoremref{thrm:wls}. \theoremref{thrm:wls} provides the non-asymptotic property of WLS and we analyze it in detail in \sectionref{sect:WLS}.
\end{proof}

\subsubsection{Error Caused by Imperfectly Learned Weights}
\begin{proposition} \propositionlabel{prop:error-weightes}
    Suppose \assumptionref{assum:target-weighting} (with parameter $\boundlambda$) and \assumptionref{assum:estimate-weighting} (with parameter $\epsilon$) hold. Then
    $$
    \|\boldbeta_{\hat{w}}-\boldbeta_w\|_2 \le \frac{\epsilon\|\Sigma_w\|_2\|\boldbeta_w\|_2}{\boundlambda - \epsilon \ \sqrt{\mathbb{E}\left[\|\boldX\|_2^4\right]}}\left(\frac{\sqrt{\mathbb{E}\left[\|\boldX\|_2^4\right]}}{\|\Sigma_w\|_2}+\frac{\sqrt{\mathbb{E}\left[\|\boldX Y\|_2^2\right]}}{\|\bbE[w(\boldX)\boldX Y]\|_2}\right)
    $$
\end{proposition}
\begin{proof}
    Let $\Delta_w (\boldX) \triangleq \hat{w}(\boldX) - w(\boldX)$ and $\Delta_b  \triangleq \mathbb{E}\left[\Delta_w(\boldX)\boldX Y\right]$. We can prove that
    $$
    \begin{aligned}
        \|\Delta_b\|_2 & = \|\mathbb{E}\left[\Delta w(\boldX)\boldX Y\right]\|_2 \\
        & \le \mathbb{E} \left[\Delta w(\boldX)\|\boldX Y\|_2\right] && (\text{triangle inequality of norms})\\
        & \le \sqrt{\mathbb{E}[\Delta w(\boldX)^2]\mathbb{E}\left[\|\boldX Y\|_2^2\right]} && \text{(Cauchy–Schwarz inequality)} \\
        & = \epsilon \sqrt{\mathbb{E}\left[\|\boldX Y\|_2^2\right]} && (\mathbb{E}\left[\Delta w(\boldX)^2\right] = \epsilon) \\
    \end{aligned}
    $$
    In addition, $\Sigma_w\boldbeta_w = \bbE[w(\boldX)\boldX Y]$ and $(\Sigma_w + \Delta_\Sigma)\boldbeta_{\hat{w}} = \bbE[w(\boldX)\boldX Y] + \Delta_b$. As a result, according to \lemmaref{lemma:linear-systems} and \propositionref{prop:spectral-lower-bound},
    $$
    \begin{aligned}
        \frac{\|\boldbeta_{\hat{w}}-\boldbeta_w\|_2}{\|\boldbeta_w\|_2} & \le \frac{\|\Sigma_w\|_2\|\Sigma_w^{-1}\|_2}{1-\|\Sigma_w^{-1}\|_2\|\Delta_\Sigma\|_2}\left(\frac{\|\Delta_\Sigma\|_2}{\|\Sigma_w\|_2}+\frac{\|\Delta_b\|_2}{\|\bbE[w(\boldX)\boldX Y]\|_2}\right) \\
        & \le \frac{\epsilon\|\Sigma_w\|_2}{\boundlambda - \epsilon \ \sqrt{\mathbb{E}\left[\|\boldX\|_2^4\right]}}\left(\frac{\sqrt{\mathbb{E}\left[\|\boldX\|_2^4\right]}}{\|\Sigma_w\|_2}+\frac{\sqrt{\mathbb{E}\left[\|\boldX Y\|_2^2\right]}}{\|\bbE[w(\boldX)\boldX Y]\|_2}\right)
    \end{aligned}
    $$
\end{proof}

\subsection{Proof of \texorpdfstring{\propositionref{prop:unique-boundary-all-blankets}}{Proposition}}
The proof is based on the following lemma.
\begin{lemma} [Intersection Property] \lemmalabel{lemma:intersection}
    Under \assumptionref{assum:positive}, let $\boldV_1$, $\boldV_2$, and $\boldS$ be subset of $\boldX$.
    Then,
    $$
    Y \perp \boldV_1 \mid (\boldS \cup \boldV_2) \,\,\, \& \,\,\, Y \perp \boldV_2 \mid (\boldS \cup \boldV_1) \Longrightarrow Y \perp (\boldV_1 \cup \boldV_2) \mid \boldS.
    $$
\end{lemma}
The proof of \lemmaref{lemma:intersection} can be found in \citep[Section 3.1.2]{pearl2014probabilistic}. Now we prove the original theorem.
\begin{proof}[Proof of \propositionref{prop:unique-boundary-all-blankets}]
    According to \citet{statnikov2013algorithms}, if the distribution $\ptrain$ satisfies the intersection property, then there exists a unique Markov boundary of $Y$.

    Next we prove the exact form of the Markov blankets. On the one hand, from \lemmaref{lemma:intersection}, we can know that under \assumptionref{assum:positive}, if $\boldS_1$ and $\boldS_2$ are Markov blankets of $Y$, so does $\boldS_1 \cap \boldS_2$. As a result, for any $\boldS \in \mbl(Y)$, we have $\boldS \cap \mbd(Y) \in \mbl(Y)$. Because $\mbd(Y)$ is the minimal element in $\mbl(Y)$, we have $|\boldS \cap \mbd(Y)| \ge |\mbd(Y)|$. Hence, $\mbd(Y) \subseteq \boldS$.
    
    On the other hand, for any $\boldS$ that satisfies $\mbd(Y) \subseteq \boldS \subseteq \boldX$. Let $\boldV = \boldX \backslash \boldS$ and $\boldW = \boldS \backslash \mbd(Y)$. Then
    $$
    \begin{aligned}
        \ptrain(Y, \boldV | \boldS) = & \frac{\ptrain(Y, \boldV, \mbd(Y), \boldW)}{\ptrain(\boldS)} = \frac{\ptrain(Y, \boldV, \boldW | \mbd(Y))\ptrain(\mbd(Y))}{\ptrain(\boldS)} \\
        = & \frac{\ptrain(Y | \mbd(Y))\ptrain(\boldV, \boldW | \mbd(Y))\ptrain(\mbd(Y))}{\ptrain(\boldS)} \\
        = & \frac{\ptrain(Y|\mbd(Y))\ptrain(\boldV, \boldW, \mbd(Y))}{\ptrain(\boldS)} \\
        = & \frac{\ptrain(Y | \mbd(Y)) \ptrain(\boldV, \boldS)}{\ptrain(\boldS)} =  \ptrain(Y | \boldS)\ptrain(\boldV | \boldS).
    \end{aligned}
    $$
    As a result, $Y \perp \boldV \mid \boldS$ and $\boldS$ is a Markov blanket of $Y$. To conclude, $\mbl(Y) = \{\boldS \subseteq \boldX \mid \mbd(Y) \subseteq \boldS\}$.
\end{proof}

\subsection{Proof of \texorpdfstring{\theoremref{thrm:subset}}{Theorem}}
\begin{proof}
    $\forall \boldS \in \mbl(Y)$, $Y \perp (\boldX \backslash \boldS) \mid \boldS$. Hence $\mathbb{E}[Y | \boldX] = \mathbb{E}[Y | \boldS]$ and $\boldS \in \pmbl(Y)$, which implies $\mbl(Y) \subseteq \pmbl(Y)$.
    
    Therefore, $\forall \boldS \in \mbl(Y)$, $\boldS \in \pmbl(Y)$. According to \theoremref{thrm:unique-stable-set}, $\pmbd(Y) \subseteq \boldS$. 
    In particular, let $\boldS = \mbd(Y) \in \mbl(Y)$ and we have $\pmbd(Y) \subseteq \mbd(Y)$.
\end{proof}

\subsection{Proof of \texorpdfstring{\theoremref{thrm:markov-ood}}{Theorem}}
The proof is based on the following proposition.
\begin{proposition} [\citet{statnikov2013algorithms, strobl2016markov}] \propositionlabel{prop:markov-iid}
    Suppose $\mathbb{M}$ is a performance metric that is maximized only when $P(Y | \boldX)$ is estimated accurately and $\mathbb{L}$ is a learning algorithm that can approximate any conditional probability distribution. Suppose $\boldS \subseteq \boldX$ is a subset of variables, then
    \begin{enum}
        \item $\boldS$ is a Markov blanket of $Y$ if and only if it is an optimal predictor of $Y$, and
        \item $\boldS$ is a Markov boundary of $Y$ if and only if it is a minimal and optimal predictor of $Y$.
    \end{enum}
\end{proposition}

Now we can prove the original theorem.

\begin{proof}[Proof of \theoremref{thrm:markov-ood}]
    We use $\mbltest$ and $\mbdtest$ to denote the Markov blankets and Markov boundary in the test distribution. We first prove that $\mbltest(Y) = \mbl(Y)$ and $\mbdtest(Y) = \mbd(Y)$.
    
    Suppose $\boldS$ is a Markov blanket under the training distribution $\ptrain$. Let $\boldV = \boldX \backslash \boldS$. Under \assumptionref{assum:covariate-shift} and \assumptionref{assum:positive}, $\forall \boldv \in \calV, \bolds \in \calS, y \in \calY$,
    $$
    \begin{aligned}
        \ptest(Y = y | \boldV = \boldv, \boldS = \bolds) = \ptrain(Y = y | \boldV = \boldv, \boldS = \bolds) = \ptrain(Y = y | \boldS = \bolds).
    \end{aligned}
    $$
    Hence,
    $$
    \begin{aligned}
        & \ptest(Y = y | \boldS = \bolds) \\
        = & \int_{\calV} \ptest(Y = y | \boldV = \boldv', \boldS=\bolds)\ptest(\boldV=\boldv'|\boldS=\bolds)\mathrm{d}\boldv' \\
        = & \int_{\calV} \ptrain(Y = y | \boldS=\bolds)\ptest(\boldV=\boldv'|\boldS=\bolds)\mathrm{d}\boldv' \\
        = & \ptrain(Y = y | \boldS=\bolds) = \ptest(Y = y | \boldV = \boldv, \boldS = \bolds).
    \end{aligned}
    $$
    As a result, $\boldS$ is a Markov blanket under $\ptest$, which implies $\mbl(Y) \subseteq \mbltest(Y)$. With similar calculations, we can show that $\mbltest(Y) \subseteq \mbl(Y)$, which finally shows that $\mbltest(Y) = \mbl(Y)$. Because Markov boundary is the minimal element of the set of Markov blankets, we can get that $\mbdtest(Y) = \mbd(Y)$.
    
    Now the claim follows from \propositionref{prop:markov-iid}.
\end{proof}

\section{Non-asymptotic Property of WLS} \sectionlabel{sect:WLS}
\subsection{Main Result}
\begin{condition} [Bounded statistical leverage] \conditionlabel{cond:bounded-statistical}
    For a weighting function $w \in \calW$, there exists a finite constant $\rho_w \ge 1$, such that, in the training distribution $\ptrain$, almost surely,
    \begin{equation*}
        \frac{\sqrt{w(\boldX)}\left\|\Sigma_w^{-1/2}\boldX\right\|_2}{\sqrt{d}} \le \rho_w.
    \end{equation*}
\end{condition}

\begin{condition} [Bounded approximation error] \conditionlabel{cond:bounded-approx}
    For a weighting function $w \in \calW$, there exists a finite constant $b_w \ge 0$ such that, in the training distribution $\ptrain$, almost surely,
    \begin{equation*}
        \frac{w(\boldX)\left\|\Sigma_w^{-1/2}\appr(\boldX)\boldX\right\|_2}{\sqrt{d}} \le b_w.
    \end{equation*}
\end{condition}

\begin{condition} [Noise] \conditionlabel{cond:sub-gaussian}
    For a weighting function $w \in \calW$, there exists a finite constant $\sigma_w \ge 0$ such that, in the training distribution $\ptrain$, almost surely,
    \begin{equation*}
        \forall \eta \in \bbR, \quad \bbE\left[\left. \exp\left(\eta \sqrt{w(\boldX)}\noise(\boldX)\right) \right| \boldX\right] \le \exp\left(\frac{\eta^2\sigma_w^2}{2}\right).
    \end{equation*}
\end{condition}

\begin{theorem} \theoremlabel{thrm:wls}
    For a weighting function $w \in \mathcal{W}$. Pick any $t > \max\{0, 2.6 - \log d\}$. Suppose $w$ satisfies \conditionref{cond:bounded-statistical} (with parameter $\rho_w$), \conditionref{cond:bounded-approx} (with parameter $b_w$), and \conditionref{cond:sub-gaussian} (with parameter $\sigma_w$) and that
    \begin{equation*}
        n \ge 6\rho_w^2d(\log d + t).
    \end{equation*}
    With probability at least $1 - 3e^{-t}$,
    $$
    \left\|\hat{\boldbeta}_w - \boldbeta_w\right\|_{\Sigma_w}^2 \le \frac{2\sigma_w^2\left(d+2\sqrt{td}+2t\right)}{n} + \frac{4\rho_w^2d\cdot \bbE[w(\boldX)\appr(X)^2]}{n} (1 + \sqrt{8t})^2 + o(1/n).
    $$
\end{theorem}
\begin{remark}
    The constant $b_w$ only appears in $o(1/n)$ terms.
\end{remark}

\subsection{Proof}
The main scope of the proof follows \citet{hsu2014random}, which provides the non-asymptotic properties of OLS and ridge regression. We further adapt it to the WLS here. We use $\bbE[\cdot]$ to denote $\bbE_{\ptrain}[\cdot]$ throughout the section.

Let
$$
\bar{\boldbeta}_w \triangleq \hat{\Sigma}_w^{-1}\hat{\bbE}[w(\boldX)\boldX\bbE[Y | \boldX]].
$$
Then
$$
\left\|\hat{\boldbeta}_w - \boldbeta_w\right\|_{\Sigma_w}^2 \le \left(\left\|\bar{\boldbeta}_w - \hat{\boldbeta}_w\right\|_{\Sigma_w} + \left\|\bar{\boldbeta}_w - \boldbeta_w\right\|_{\Sigma_w}\right)^2 \le 2\left(\left\|\bar{\boldbeta}_w - \hat{\boldbeta}_w\right\|_{\Sigma_w}^2 + \left\|\bar{\boldbeta}_w - \boldbeta_w\right\|_{\Sigma_w}^2\right)
$$
We analyze the two terms $\left\|\bar{\boldbeta}_w - \hat{\boldbeta}_w\right\|_{\Sigma_w}^2$ and $\left\|\bar{\boldbeta}_w - \boldbeta_w\right\|_{\Sigma_w}^2$ separately and the result is a straightforward combination of \propositionref{prop:spectral-norm}, \propositionref{prop:approx-error}, and \propositionref{prop:noise}. We first define the following $\Delta$.
$$
\Delta \triangleq \Sigma_w^{-1/2}(\hat{\Sigma}_w - \Sigma_w)\Sigma_w^{-1/2},
$$

\subsubsection{Effect of errors in $\hat{\Sigma}_w$}
\begin{proposition} [Spectral norm error in $\hat{\Sigma}_w$] \propositionlabel{prop:spectral-norm}
    Suppose $w$ satisfies \conditionref{cond:bounded-statistical} (with parameter $\rho_w$) holds. Pick $t > \max\{0, 2.6 - \log d\}$. With probability at least $1 - e^{-t}$,
    $$
        \|\Delta\|_2 \le \sqrt{\frac{4\rho_w^2 d(\log d + t)}{n}} + \frac{2\rho_w^2 d (\log d + t)}{3n}.
    $$
\end{proposition}
\begin{proof}
    First, define $\tildeX \triangleq \sqrt{w(\boldX)}\Sigma_w^{-1/2}\boldX$ and let
    $$
    \boldZ \triangleq \tildeX \tildeX^T - I = \Sigma_w^{-1/2}\left(w(\boldX)\boldX\boldX^T-\Sigma_w\right)\Sigma_w^{-1/2}.
    $$
    So $\Delta=\hat{\bbE}[\boldZ]$. Observe that $\bbE[\boldZ]=0$, and
    $$
    \|\boldZ\|_2 = \max\{\lambda_\text{max}(\boldZ), \lambda_\text{max}(-\boldZ)\} \le \max\{\|\tildeX\|_2^2, 1\} \le \rho_w^2d.
    $$
    Here the second inequality is based on \conditionref{cond:bounded-statistical}. Moreover,
    $$
    \bbE\left[\boldZ^2\right] = \bbE\left[\left(\tildeX\tildeX^T\right)^2\right] - I = \bbE\left[\|\tildeX\|_2^2\left(\tildeX\tildeX^T\right)\right] - I. 
    $$
    As a result,
    $$
    \begin{aligned}
        \lambda_\text{max}\left(\bbE\left[\boldZ^2\right]\right) & \le \lambda_\text{max}\left(\bbE\left[\|\tildeX\|_2^2\left(\tildeX\tildeX^T\right)\right]\right) \le \rho_w^2d \cdot \lambda_\text{max}(I) \le \rho_w^2d, \\
        \text{tr}\left(\bbE\left[\boldZ^2\right]\right) & \le \text{tr}\left(\bbE\left[\|\tildeX\|_2^2\left(\tildeX\tildeX^T\right)\right]\right) \le \rho_w^2d \cdot \text{tr}(I) = \rho_w^2d^2.
    \end{aligned}
    $$
    The proposition now follows from \lemmaref{lemma:matrix-bound}.
\end{proof}

\begin{proposition} [Relative spectral norm error in $\hat{\Sigma}_w$ \citep{hsu2014random}] \propositionlabel{prop:relative-spectral}
    If $\|\Delta\|_2 < 1$, then
    $$
    \left\|\Sigma_w^{1/2}\hat{\Sigma}_w^{-1}\Sigma_w^{1/2}\right\|_2 \le \frac{1}{1 - \|\Delta\|_2}.
    $$
\end{proposition}
\begin{proof}
    Observe that,
    $$
    \Sigma_w^{-1/2}\hat{\Sigma}_w\Sigma_w^{-1/2} = \Sigma_w^{-1/2}\left(\Sigma_w + \hat{\Sigma}_w - \Sigma_w\right)\Sigma_w^{-1/2} = I + \Delta,
    $$
    and according to the assumption $\|\Delta\|_2 < 1$ and Weyl's theorem \citep{horn2012matrix}, we have
    $$
    \lambda_{\min}\left(I + \Delta\right) \ge 1 - \|\Delta\|_2 > 0.
    $$
    Therefore,
    $$
    \left\|\Sigma_w^{1/2}\hat{\Sigma}_w^{-1}\Sigma_w^{1/2}\right\|_2 = \lambda_{\max}\left(\left(\Sigma_w^{-1/2}\hat{\Sigma}_w\Sigma_w^{-1/2}\right)^{-1}\right) = \lambda_{\max}\left(\left(I + \Delta\right)^{-1}\right) = \frac{1}{\lambda_{\min}\left(I + \Delta\right)} \le \frac{1}{1 - \|\Delta\|_2}.
    $$
\end{proof}

\subsubsection{Effect of approximation error}
\begin{proposition} \propositionlabel{prop:approx-error}
    Suppose $w$ satisfies \conditionref{cond:bounded-statistical} (with parameter $\rho_w$) and \conditionref{cond:bounded-approx} (with parameter $b_w$) hold. Pick any $t > 0$. If $\|\Delta\|_2 < 1$, then
    $$
    \|\bar{\boldbeta}_w - \boldbeta_w\|_{\Sigma_w} \le \frac{1}{1 - \|\Delta\|_2}\left\|\hat{\bbE}[w(\boldX)\appr(\boldX)\boldX]\right\|_{\Sigma_w^{-1}}.
    $$
    Moreover, with probability at least $1 - e^{-t}$,
    $$
    \begin{aligned}
        \left\|\hat{\bbE}[w(\boldX)\appr(\boldX)\boldX]\right\|_{\Sigma_w^{-1}} & \le \sqrt{\frac{\bbE\left[\left\|\Sigma_w^{-1/2}w(\boldX)\appr(\boldX)\boldX\right\|_2^2\right]}{n}} (1 + \sqrt{8t}) + \frac{4b_wt\sqrt{d}}{3n} \\
        & \le \sqrt{\frac{\rho_w^2d\cdot \bbE[w(\boldX)\appr(\boldX)^2]}{n}} (1 + \sqrt{8t}) + \frac{4b_wt\sqrt{d}}{3n}
    \end{aligned}
    $$
\end{proposition}
\begin{proof}
    By definition,
    $$
    \begin{aligned}
        \bar{\boldbeta}_w - \boldbeta_w & = \hat{\Sigma}_w^{-1}\left(\hat{\bbE}[w(\boldX)\boldX \bbE[Y|\boldX]] - \hat{\Sigma}_w \boldbeta_w\right) \\
        & = \Sigma_w^{-1/2}\left(\Sigma_w^{1/2}\hat{\Sigma}_w^{-1}\Sigma_w^{1/2}\right)\Sigma_w^{-1/2}\left(\hat{\bbE}[w(\boldX)\boldX(\langle\boldbeta_w, \boldX\rangle + \appr(\boldX))] - \hat{\Sigma}_w\boldbeta_w\right) \\
        & = \Sigma_w^{-1/2}\left(\Sigma_w^{1/2}\hat{\Sigma}_w^{-1}\Sigma_w^{1/2}\right)\Sigma_w^{-1/2}\hat{\bbE}[w(\boldX)\appr(\boldX)\boldX].
    \end{aligned}
    $$
    Therefore, with the submultiplicative property of the spectral norm,
    $$
    \begin{aligned}
    \|\bar{\boldbeta}_w - \boldbeta_w\|_{\Sigma_w} & \le \left\|\Sigma_w^{1/2}\Sigma_w^{-1/2}\right\|_2\left\|\Sigma_w^{1/2}\hat{\Sigma}_w^{-1}\Sigma_w^{1/2}\right\|_2\left\|\hat{\bbE}[w(\boldX)\appr(\boldX)\boldX]\right\|_{\Sigma_w^{-1}} \\
    & \le \frac{1}{1 - \|\Delta\|_2}\left\|\hat{\bbE}[w(\boldX)\appr(\boldX)\boldX]\right\|_{\Sigma_w^{-1}}.
    \end{aligned}
    $$
    Here the last inequality is according to \propositionref{prop:relative-spectral}.

    Now prove the second part of the claim. Observe that
    $$
    \begin{aligned}
        \bbE[w(\boldX)\appr(\boldX)\boldX] & = \bbE[w(\boldX)\boldX(\bbE[Y | \boldX] - \langle\boldbeta_w, \boldX\rangle)] \\
        & = \bbE\left[w(\boldX)\boldX\bbE[Y | \boldX]\right] - \bbE[w(\boldX)\boldX\langle\boldbeta_w, \boldX\rangle] \\
        & = 0.
    \end{aligned}
    $$
    Therefore,
    $$
    \bbE\left[\Sigma_w^{-1/2}w(\boldX)\appr(\boldX)\boldX\right] = \Sigma_w^{-1/2}\bbE[w(\boldX)\appr(\boldX)\boldX] = 0.
    $$
    In addition, according to \conditionref{cond:bounded-approx},
    $$
    \left\|\Sigma_w^{-1/2}w(\boldX)\appr(\boldX)\boldX\right\|_2 \le b_w\sqrt{d}.
    $$
    Moreover, by \conditionref{cond:bounded-statistical}
    $$
    \bbE\left[\left\|\Sigma_w^{-1/2}w(\boldX)\appr(\boldX)\boldX\right\|_2^2\right] = \bbE\left[w(\boldX)\appr(\boldX)^2\left\|\Sigma_w^{-1/2}\sqrt{w(\boldX)}\boldX\right\|_2^2\right] \le \rho_w^2d\cdot\bbE[w(\boldX)\appr(\boldX)^2].
    $$
    The claim now follows from \lemmaref{lemma:vector-bound}.
\end{proof}

\subsubsection{Effect of noise}
\begin{proposition} \propositionlabel{prop:noise}
    Suppose $w$ satisfies \conditionref{cond:sub-gaussian} (with parameter $\sigma_w$) holds. Pick any $t > 0$. With probability at least $1 - e^{-t}$, either $\|\Delta\|_2 \ge 1$, or
    $$
    \|\Delta\|_2 < 1 \quad \text{and} \quad \left\|\bar{\boldbeta}_w - \hat{\boldbeta}_w\right\|_{\Sigma_w}^2 \le \frac{1}{1 - \|\Delta\|_2}\frac{\sigma_w^2\left(d+2\sqrt{td}+2t\right)}{n}.
    $$
\end{proposition}
\begin{proof}
    Observe that
    $$
    \left\|\bar{\boldbeta}_w - \hat{\boldbeta}_w\right\|_{\Sigma_w}^2 \le \left\|\Sigma_w^{1/2}\hat{\Sigma}_w^{-1/2}\right\|_2^2\left\|\bar{\boldbeta}_w - \hat{\boldbeta}_w\right\|_{\hat{\Sigma}_w}^2 = \left\|\Sigma_w^{1/2}\hat{\Sigma}_w^{-1}\Sigma_w^{1/2}\right\|_2\left\|\bar{\boldbeta}_w - \hat{\boldbeta}_w\right\|_{\hat{\Sigma}_w}^2.
    $$
    According to \propositionref{prop:relative-spectral}, if $\|\Delta\|_2 < 1$, then $\left\|\Sigma_w^{1/2}\hat{\Sigma}_w^{-1}\Sigma_w^{1/2}\right\|_2 \le 1 / (1 - \|\Delta\|_2)$.

    Let $\boldxi \triangleq (\sqrt{w(\boldx^{(1)})}\noise(\boldx^{(1)}), \sqrt{w(\boldx^{(2)})}\noise(\boldx^{(2)}), \dots, \sqrt{w(\boldx^{(n)})}\noise(\boldx^{(n)}))$ be the random vector and $\noise(\boldx^{(i)}) = \left(y^{(i)} - \bbE[Y | \boldX = \boldx^{(i)}]\right)$. By the definition of $\hat{\boldbeta}_w$ and $\bar{\boldbeta}_w$,
    $$
    \left\|\bar{\boldbeta}_w - \hat{\boldbeta}_w\right\|_{\hat{\Sigma}_w}^2 = \left\|\hat{\Sigma}_w^{-1/2}\hat{\bbE}\left[w(\boldX)\boldX(\bbE[Y | \boldX] - Y)\right]\right\|_2^2 = \boldxi^T\hat{K}\boldxi,
    $$
    where $\hat{K} \in \bbR^{n \times n}$ is a symmetric matrix whose $(i, j)$-th entry is
    $$
    \hat{K}_{i,j}\triangleq 1 / n^2\left\langle\hat{\Sigma}_w^{-1/2}\sqrt{w(\boldx^{(i)})}\boldx^{(i)}, \hat{\Sigma}_w^{-1/2}\sqrt{w(\boldx^{(j)})}\boldx^{(j)}\right\rangle.
    $$
    According to the proof of Lemma 6 \citep{hsu2014random}, the nonzero eigenvalues of $\hat{K}$ are the same as those of
    $$
    \frac{1}{n}\hat{\bbE}\left[\left(\hat{\Sigma}_w^{-1/2}\sqrt{w(\boldX)}\boldX\right)\left(\hat{\Sigma}_w^{-1/2}\sqrt{w(\boldX)}\boldX\right)^T\right] = \frac{1}{n}\hat{\Sigma}_w^{-1/2}\hat{\Sigma}_w\hat{\Sigma}_w^{-1/2} = \frac{1}{n}I_d,
    $$
    where $I_d$ is the identity matrix with dimension $d$. By \lemmaref{lemma:sub-gaussian}, with probability at least $1 - e^{-t}$ (conditioned on $\boldx^{(1)}, \boldx^{(2)}, \dots, \boldx^{(n)}$),
    $$
    \boldxi^T\hat{K}\boldxi \le \sigma_w^2\left(\text{tr}(\hat{K})+2\sqrt{\text{tr}(\hat{K}^2)t} + 2\left\|\hat{K}\right\|_2t\right) \le \frac{\sigma_w^2\left(d+2\sqrt{td}+2t\right)}{n}.
    $$
    Now the claim follows.
\end{proof}

\section{Important lemmas}
\begin{lemma} [\citet{chandrasekaran1995sensitivity}] \lemmalabel{lemma:linear-systems}
    Suppose $Ax = b$ and $\hat{A}\hat{x}=\hat{b}$. Suppose $\|A^{-1}\|\|A - \hat{A}\| < 1$, then
    $$
        \frac{\|x-\hat{x}\|}{\|x\|} \le \frac{\|A\|\|A^{-1}\|}{1-\|A^{-1}\|\|A - \hat{A}\|}\left(\frac{\|A-\hat{A}\|}{\|A\|}+\frac{\|b-\hat{b}\|}{\|b\|}\right).
    $$
\end{lemma}

\begin{lemma} [Matrix Bernstein bound \citep{hsu2012taila}] \lemmalabel{lemma:matrix-bound}
    Let $A$ be a random matrix, and $r > 0$, $v > 0$, and $k > 0$ be such that, almost surely,
    $$
        \bbE[A] = 0, \lambda_\text{max}(A) \le r, \lambda_\text{max}\left(\bbE\left[A^2\right]\right) \le v, \text{tr}\left(\bbE[A^2]\right) \le vk.
    $$
    If $A_1, A_2, \dots, A_n$ are independent copies of $A$, then for any $t > 0$,
    $$
    \pr\left[\lambda_\text{max}\left(\frac{1}{n}\sum_{i=1}^nA_i\right) > \sqrt{\frac{2vt}{n}} + \frac{rt}{3n}\right] \le kt\left(e^t - t - 1\right)^{-1}.
    $$
    If $t > 2.6$, then $t\left(e^t - t - 1\right)^{-1} \le e^{-t/2}$.
\end{lemma}

\begin{lemma} [Vector Bernstein bound \citep{hsu2012tailb}] \lemmalabel{lemma:vector-bound}
    Let $\boldx^{(1)}, \boldx^{(2)}, \dots, \boldx^{(n)}$ be independent random vectors such that
    $$
    \sum_{i=1}^n \bbE\left[\left\|\boldx^{(i)}\right\|_2^2\right] \le v \quad \text{and} \quad \left\|\boldx^{(i)}\right\|_2 \le r
    $$
    for all $i=1,2, \dots, n$, almost surely. Let $\bolds \triangleq \boldx^{(1)} + \boldx^{(2)} + \cdots + \boldx^{(n)}$. For all $t > 0$,
    $$
    \pr\left[\|\bolds\|_2 > \sqrt{v}(1 + \sqrt{8t}) + (4/3)rt\right] \le e^{-t}.
    $$
\end{lemma}

\begin{lemma} [Quadratic forms of a sub-Gaussian random vector \citep{hsu2012tailb}] \lemmalabel{lemma:sub-gaussian}
    Let $\boldxi$ be a random vector taking values from $\bbR^n$ such that for some $c \ge 0$,
    $$
    \bbE\left[\exp(\langle \boldu, \boldxi\rangle)\right] \le \exp\left(c\|\boldu\|_2^2/2\right), \quad \forall \boldu \in \bbR^n.
    $$
    For all symmetric positive semidefinite matrices $K \succeq 0$, and all $t > 0$,
    $$
    \pr\left[\boldxi^TK\boldxi > c\left(\text{tr}(K) + 2\sqrt{\text{tr}(K^2)t} + 2 \|K\|_2t\right)\right] \le e^{-t}.
    $$
\end{lemma}